\documentclass[twoside,11pt]{article}

\usepackage{blindtext}

\usepackage[abbrvbib, preprint]{jmlr2e}

\usepackage{my-macros}

\DeclareMathOperator{\supp}{supp}
\renewcommand{\parallel}{\mathbin{\!/\mkern-2mu/\!}}

\usepackage[dvipsnames]{xcolor}

\usepackage{lastpage}
\jmlrheading{25}{2024}{1-\pageref{LastPage}}{5/23; Revised
4/24}{5/24}{23-0651}{Tengyuan Liang}
\ShortHeadings{Blessings and Curses of Covariate Shifts}{Liang}

\firstpageno{1}

\begin{document}

\title{Blessings and Curses of Covariate Shifts: Adversarial Learning Dynamics, Directional Convergence, and Equilibria}

\author{\name Tengyuan Liang \email tengyuan.liang@chicagobooth.edu \\
       \addr 
	   Booth School of Business\\
	   University of Chicago\\
       Chicago, IL 60637, USA
       }

\maketitle

\begin{abstract}%   <- trailing '%' for backward compatibility of .sty file
Covariate distribution shifts and adversarial perturbations present robustness challenges to the conventional statistical learning framework: mild shifts in the test covariate distribution can significantly affect the performance of the statistical model learned based on the training distribution. The model performance typically deteriorates when extrapolation happens: namely, covariates shift to a region where the training distribution is scarce, and naturally, the learned model has little information. For robustness and regularization considerations, adversarial perturbation techniques are proposed as a remedy; however, careful study needs to be carried out about what extrapolation region adversarial covariate shift will focus on, given a learned model. This paper precisely characterizes the extrapolation region, examining both regression and classification in an infinite-dimensional setting. We study the implications of adversarial covariate shifts to subsequent learning of the equilibrium---the Bayes optimal model---in a sequential game framework. We exploit the dynamics of the adversarial learning game and reveal the curious effects of the covariate shift to equilibrium learning and experimental design. In particular, we establish two directional convergence results that exhibit distinctive phenomena: (1) a blessing in regression, the adversarial covariate shifts in an exponential rate to an optimal experimental design for rapid subsequent learning; (2) a curse in classification, the adversarial covariate shifts in a subquadratic rate to the hardest experimental design trapping subsequent learning.
\end{abstract}


\section{Introduction}
In supervised learning, a folklore rule is that the test data set should follow the same,  or at least resemble the probability distribution from which the training data set is drawn, for strong guarantees of learnability and generalization \citep{vapnik1999nature}. The reason is grounded since, if not, either (a) concept shift, namely the conditional distribution of $P(Y|X)$ changes, hence the underlying Bayes optimal prediction model $f^\star: X\rightarrow Y$ could shift, or (b) covariate shift, namely the covariate distribution $\mu \in \cP(X)$ shifts so that the underlying evaluation metric for the learned model $f$ changes.\footnote{One typical evaluation metric is $\| f - f^\star\|_{L^2(\mu)}$, where the underlying covariate distribution $\mu \in \cP(X)$ varies.} Nevertheless, supervised learning is often deployed ``in the wild,'' meaning the test data distribution typically extrapolates the training data distribution. 

Concept shift is inherently a complex problem, as if shooting for a moving target. However, covariate shift may be less severe a problem if the concept stays the same. Historically, specific statistical methods allow for mild extrapolation,\footnote{Here we mean the region of the extrapolation is still contained in the region of the seen data, but the test distribution can differ from the training distribution.} say the (fixed-design) linear regression and the (local) nonparametric regression \citep{stone1980optimal}. Recently, a few notable lines of work have arisen to study the covariate shift. \citet{shimodaira2000improving}, \citet{sugiyama2007covariate}, and \citet{sugiyama2005generalization} studied covariate shift adaption assuming knowledge of the density ratio of the covariate shift. \cite{ben2006analysis} and \cite{blitzer2007learning} initiated the learning-theoretic study of domain adaptation. There, the generalization error on the target/test domain is bounded by the standard generalization error on the source/training domain, plus a discrepancy term---for instance, total variation or its tighter analog induced by the hypothesis class---quantifying the covariate shift between training and test distributions. Later, on the one hand, a collection of research extended the theory to allow for concept shifts with general loss functions by proposing other notions of discrepancy measures or quantities contrasting two distributions \citep{mansour2009domain, ben2010theory, mohri2012new, kpotufe2018marginal}. On the other hand, by establishing lower bounds, \cite{david2010impossibility} studied assumptions on the relationship between training and test distributions necessary for successful domain adaptation. The learning-theoretic framework brought forth new domain adaptation algorithms, for instance, reweighing the empirical distribution to minimize the discrepancy between source and target \citep{mansour2009domain}, 
and finding common representation space with small discrepancy while maintaining good performance on the training data \citep{ben2006analysis,ganin2016domain}. A considerable body of domain adaptation literature primarily focuses on when the conditional relationship $P(Y|X)$ is invariant, and thus the Bayes optimal model stays fixed, yet the covariate distribution $P(X)$ shifts. We follow this tradition of an invariant Bayes optimal model and investigate adversarial perturbations to the covariate distribution.

The quest for robust domain adaptation also prompted the recent development in adversarial learning \citep{ganin2016domain, goodfellow2014explaining, ilyas2019adversarial, bubeck2021single}. Akin to covariate shift, adversarial perturbations have recently revived interest in machine learning and robust optimization communities \citep{goodfellow2014explaining, delage2010distributionally}. Adversarial perturbation is motivated by the following observation: small local perturbations to the covariate distribution can significantly compromise the supervised learning performance \citep{goodfellow2014explaining, madry2017towards, javanmard2022PreciseStatistical}. For example, given a supervised learning model $f$, adversarial perturbations shift the covariate distribution $\mu \in \cP(X)$ locally under a specific metric on the measure space $\cP(X)$, such that it makes the model suffer the most in predictive performance. The familiar reader will immediately identify the minimax game between the supervised learning model $f$ and the covariate distribution $\mu$: the model $f$ minimizes the risk from a model class, yet the covariate distribution $\mu$ maximizes the risk from a probability distribution class. Such a game perspective between the learning model $f$ and data distribution $\mu$ has been influential since the seminal work of boosting \citep{freund1997decision,freund1999adaptive}. Inspired by the above, we study covariate shift from a game-theoretic perspective. The connection between boosting and adversarial perturbation will be elaborated later; in a nutshell, instead of taking Kullback-Leibler divergence as a metric, we use the Wasserstein metric for adversarial perturbation, thus allowing for extrapolation outside the current covariate support.

This paper studies a particular form of covariate shifts following adversarial perturbations, with the underlying concept (i.e., Bayes optimal model) held fixed. We take a game-theoretic view to examine covariate shifts and discover curious insights. We study both regression and classification in an infinite-dimensional setting. As hinted, we will exploit the dynamics of the adversarial learning game and reveal the curious effects of the covariate shift to subsequent learning and experimental design. The models we study are discriminative in nature rather than generative\footnote{Here discriminative refers to modeling $Y|X$, and generative refers to modeling $X|Y$.} for invariance considerations: for the latter, the underlying concept $P(Y|X)$ could shift as a result of adversarial perturbations on covariates.

Now we are ready to state the main goal of this paper:
\begin{quote}
	\it
	 Adversarial covariate shifts move the current covariate domain to an extrapolation region. We precisely characterize the extrapolation region and, subsequently, the implications of adversarial covariate shifts to subsequent learning of the equilibrium, the Bayes optimal model. 
\end{quote}
Curiously, we show two directional convergence results that exhibit distinctive phenomena: (1) a blessing in regression, the adversarial covariate shifts in an exponential rate to an optimal experimental design for rapid subsequent learning, (2) a curse in classification, the adversarial covariate shifts in a subquadratic rate to the hardest experimental design trapping subsequent learning. The theoretical results will be later coupled with numerical validations. The theoretical study is admittedly based on idealized models to demonstrate clean new insights and curious dichotomy on covariate shift and adversarial learning; potential future directions will be discussed in the last section. Before diving into the problem setup, we elaborate on the background and some related literature. 

\subsection{Background and Literature Review}
We first fix some notations to make the discussions on covariate shift and adversarial perturbation concrete. Let $X$ be space for the covariates, and $Y \subset \R$ be the space for a real-valued response variable. When a pair of covariate and response data $(x, y) \in X \times Y$ is generated based on the probability measure $\pi \in \cP(X \times Y)$, we denote $(x, y) \sim \pi$. Let $f: X \rightarrow Y$ be a statistical model and $\ell(\cdot, \cdot): \R \times \R \rightarrow \R$ be a risk or loss function $(f(x), y) \mapsto \ell(f(x), y)$ that quantifies how the model $f$ performs on the data pair $(x, y)$. 

Given a statistical model $f$ and a probability measure for data set $\pi$, one can define the utility function that accesses the risk of model $f$ on data $\pi$
\begin{align*}
	\cR(f, \pi) := \E_{(x, y) \sim \pi}\big[\ell(f(x), y)\big] \;.
\end{align*}

\paragraph{Models, covariate distributions, and Bayes optimality}
Following the literature \citep{shimodaira2000improving, quinonero2008dataset}, we consider the covariate shift but not the concept shift. For a valid marginal probability measure for the covariate $\mu \in \cP(X)$, we define the induced joint measure for the covariate and response pair 
\begin{align}
	\label{eqn:disintegration}
	\pi_{\mu}:= \int \delta_x \otimes \pi_x^\star \dd{\mu(x)} \in \cP(X \times Y)\;, 
\end{align}
where $\pi_x^\star \in \cP(Y)$ denotes a fixed conditional data generating process for $\by|\bx = x$ that does not vary with $\mu$, and $\delta_{x}$ denotes the delta measure at point $x$. Equation~\eqref{eqn:disintegration} should be read as disintegration of measure \citep{villani2021topics}, meaning for all bounded continuous function $h \in C_{b}(X \times Y)$
\begin{align*}
	\int_{X \times Y} h(x, y) \dd{\pi_{\mu}(x,y)} = \int_X \big[ \int_Y h(x, y) \dd{\pi_x^\star(y)} \big] \dd{\mu(x)} \;.
\end{align*}

Given a fixed conditional distribution $\pi_x^\star$ and a loss function $\ell(\cdot, \cdot)$, one can thus define the Bayes optimal model (suppose for now that this map is well-defined)
\begin{align}
	\label{eqn:bayes}
	f^\star_{\mathrm{Bayes}} : x \mapsto \argmin_{y'\in Y} \int \ell(y', y) \dd{\pi^\star_x(y)} \;.
\end{align}
Observe that the Bayes optimal model does not change with $\mu$, the distribution of covariates $x$. 

Now, we can define the objective of the game between the model $f:X \rightarrow Y$ and the covariate distribution $\mu \in \cP(X)$, 
\begin{align}
	\label{eqn:model-dist-objective}
	\cU(f, \mu) := \cR(f, \pi_{\mu}) = \int_X \big[\int_Y \ell(f(x), y)\dd{\pi^\star_{x}(y)} \big] \dd{\mu(x)} \;.
\end{align}
It turns out Bayes optimal model $f^\star_{\mathrm{Bayes}}$ is an equilibrium of the game, as we shall show shortly.

Note that classical statistical learning theory studies $\cU(\widehat{f}_{\mu}, \mu)$ where $\widehat{f}_{\mu}$ is learned based on an empirical data set drawn from the same distribution $\pi_{\mu}$. However, when the covariate distribution shift to another measure $\nu$ that is different from the training data distribution $\mu$, the performance $\cU(\widehat{f}_{\mu}, \nu)$ deteriorates, see \cite{shimodaira2000improving} and \cite{quinonero2008dataset} for a review on covariate shifts. Assuming the knowledge of the density ratio $\dd{\nu}/\dd{\mu}$, importance weighting methods have been proposed as an adaptation to covariate shifts.

\paragraph{Adversarial perturbation}
The theoretical insights toward understanding adversarial perturbations have so far centered around robustness and regularization in various formulations, see \cite{xu2009robustness} (Theorem 3) for support vector machines, \cite{ross2018improving, madry2017towards, sinha2017certifying} for neural network models, and \cite{delage2010distributionally} for distributionally robust optimization. Given a metric measure space $(\cP(X), d)$, a covariate distribution $\mu \in \cP(X)$, and a current model $f$, consider the following population version of the adversarial perturbation, 
\begin{align}
	\label{eqn:adversarial-perturbation}
	\cU_{\gamma}(f, \mu) := \max_{\nu~:~ d^2(\nu, \mu) \leq \gamma}~ \cU(f, \nu) \;, 
\end{align}
where $\cU(\cdot, \cdot)$ is defined in \eqref{eqn:model-dist-objective}. 

Adversarial perturbation can be viewed as smoothly regularizing the original loss function, thus enforcing stability. To see this, consider the Wasserstein metric $W_2$; as done in \cite{sinha2017certifying} (Proposition 1), one can write the Lagrangian of \eqref{eqn:adversarial-perturbation} and characterize the coupling analytically 
{\small
\begin{align*}
	\max_{\lambda \geq 0} \min_{\nu \in \cP(X)}~ -\cU(f, \nu) + \lambda \left[ W_2^2(\nu, \mu) - \gamma \right] = \max_{\lambda \geq 0}~ \E_{x \sim \mu}\big[ \underbrace{ \min_{x' \in X}~ \big( -\cU(f, \delta_{x'}) + \frac{\lambda}{2}\|x'- x \|^2 \big)}_{\text{Moreau--Yosida regularization}} \big] - \gamma \lambda \;,
\end{align*}
}where $\delta_{x}$ denotes the delta measure at point $x$.
Both the robustness and regularization perspectives readily unveil, as the above is the Moreau--Yosida envelope of the function $-\cU(f, \delta_{x}): X \rightarrow \R$ with parameter $\lambda^{-1}$, thus serving as a smoothed regularization to the original loss. We refer the readers to \cite{sinha2017certifying} for detailed derivations.

The adversarial perturbation provides a robust notion of covariate shifts without requiring the explicit knowledge of density ratio. Perhaps more importantly, it extends to the extrapolation case when the support of $\nu$ differs from $\mu$. The literature on adversarial learning is growing too fast to give a complete review. To name a few: \cite{bubeck2021single,bartlett2021adversarial} studied adversarial examples using gradient steps for two/multi-layer ReLU networks with Gaussian weights; for regression, \cite{javanmard2020precise} studied precise tradeoffs between adversarial risk $\cU_{\gamma}(\widehat{f}, \mu)$ and standard risk $\cU(\widehat{f}, \mu)$ for a range of models $\widehat{f}$ interpolating between empirical risk minimization and adversarial risk minimization,  \cite{xing2021adversarially} studied properties of the adversarially robust estimate; for classification, \cite{bao2020calibrated} introduced surrogate losses that are calibrated with the adversarial 0-1 loss, \cite{hu2018does} identified certain failure mode of distributionally robust classification under $f$-divergences, \cite{javanmard2022PreciseStatistical} precisely characterized the adversarial 0-1 loss with Gaussian covariate distributions. 

\paragraph{Wasserstein gradient flow}
Adversarial distribution shift is inherently connected to Wasserstein gradient flow. Given a current model $f$, the covariate distribution is perturbed incrementally within a Wasserstein ball in an adversarial way, with a stepsize $\gamma \in \R_+$, 
\begin{align}
	\label{eqn:perturb-wass}
	\nu := \argmin_{\nu \in \cP(X)} ~ -\cU(f, \nu) + \frac{1}{\gamma} W_2^2(\nu, \mu) \;.
\end{align}
Denote the distribution shift map $\mathrm{Ds}_{\gamma}: x \mapsto \argmin_{x' \in X} \big( -\cU(f, \delta_{x'}) + \frac{1}{2\gamma}\|x'- x \|^2 \big)$ defined by the Moreau--Yosida envelope. Informally, such a map defines the worst-case covariate shift for the model $f$ evaluated at measure $\mu$, as the maximizer of the adversarial perturbation is attained at $$\nu = (\mathrm{Ds}_{\lambda^{-1}})_\# \mu$$ where $\lambda>0$ is the solution to the dual. In the infinitesimal limit $\gamma \asymp \lambda^{-1} \rightarrow 0$, one can show that \citep{ambrosio2005gradient, guoOnlineLearningTransport2022}
\begin{align}
	\frac{\big( \mathrm{Ds}_{\gamma} - \mathrm{Id} \big) [x]}{\gamma} \rightarrow \pdv{x}  \cU(f, \delta_{x}) \;.
\end{align}
The distribution shift map $\mathrm{Ds}_{\gamma}$ presents a way of constructing adversarial examples \citep{goodfellow2014explaining, ilyas2019adversarial, bubeck2021single}, namely a couple $x \approx x'$ such that $\cU(f, \delta_{x'})-\cU(f, \delta_{x})$ is large.

The continuous-time analog of the adversarial perturbation \eqref{eqn:perturb-wass} is called the Wasserstein gradient flow as $\gamma \rightarrow 0$, where the density $\rho_t$ (associated with $\nu_t$), w.r.t. the Lebesgue measure, evolves according to the following PDE \citep{ambrosio2005gradient}
\begin{align}
	\partial_t \rho_t + \nabla \cdot( \rho_t \mathrm{V}) = 0 \;, ~\text{where}~\mathrm{V}: x \mapsto \pdv{x} \cU(f, \delta_x) \;. 
\end{align}

\subsection{Problem Setup}
\label{sec:problem-setup}

This paper considers regression and classification problems in an infinite-dimensional setting.
Let $X \subset \R^{\mathbb{N}}$ be a subset of a possibly infinite-dimensional space for the covariates, and $Y \subset \R$ be the space for a real-valued response variable. 
We concern an infinite-dimensional linear model class $\cF: =\{f_{\theta} ~|~ f_{\theta}(x) := \langle x, \theta \rangle, \theta \in \ell^2_{\mathbb{N}} \}$ where the inner-product corresponds to the Hilbert space $\ell^2_{\mathbb{N}}$. Slightly abusing the notation, we write for convenience the utility function
\begin{align}
	\label{eqn:utility}
	\cU(\theta, \mu) = \E_{(x, y) \sim \pi_{\mu}}\big[\ell(f_{\theta}(x), y)\big] = \int_X \big[\int_Y \ell(f_\theta(x), y)\dd{\pi^\star_{x}(y)} \big] \dd{\mu(x)} \;.
\end{align}

We investigate two types of conditional relationships for $\pi^\star_x$ in \eqref{eqn:disintegration}, namely for some $\theta^\star \in \ell^2_{\mathbb{N}}$:
\begin{align*}
	&\text{Regression:}~ \by|\bx = x \sim \mathrm{Gaussian}\big(\langle x, \theta^\star\rangle, 1 \big), ~\ell(f, y) = (f- y)^2 \; ; \\
	&\text{Classification:}~ \by|\bx = x \sim \mathrm{Bernoulli}\big( \sigma(\langle x, \theta^\star\rangle) \big), ~\ell(f, y) = -f y + \log(1+e^{f}) \;.
\end{align*}
Here $\sigma(z) = 1/(1+e^{-z})$ is the sigmoid function. Note that in the classification setup, $y \in \{0, 1\}$.
In both the regression and classification settings we study, the Bayes optimal model \eqref{eqn:bayes} is uniquely defined
\begin{align*}
	f^\star_{\mathrm{Bayes}}(x)  = \langle x, \theta^\star \rangle \;.
\end{align*}

\paragraph{Game and equilibria} 
In the game between the model $\theta \in \ell^2_{\mathbb{N}}$ and the covariate distribution $\mu \in \cP(X)$, 
\begin{align}
	\min_{\theta } \max_{\mu }~ \cU(\theta, \mu) & \geq  \max_{\mu} \min_{\theta}~  \cU(\theta, \mu) \geq \max_{\mu }~ \int_X \big[\min_{\theta \in \ell^2_{\mathbb{N}}} \int_Y \ell(f_\theta(x), y)\dd{\pi^\star_{x}(y)} \big] \dd{\mu(x)} \label{eqn:minimax}\\
	&=  \max_{\mu } ~  \cU(\theta^\star, \mu) \geq \min_{\theta} \max_{\mu}~ \cU(\theta, \mu) \;. \nonumber
\end{align}
Therefore, the min-max theorem holds in this infinite-dimensional context when the Bayes optimal model $f^\star_{\mathrm{Bayes}}(x)  = \langle x, \theta^\star \rangle$ is well-defined. A Nash equilibrium of $\cU(\cdot, \cdot)$ is precisely the Bayes optimal model $f^\star_{\mathrm{Bayes}}$. In plain language, the covariate distribution shift does not affect the notion of equilibrium of the game.

The game perspective is not new. For example, the celebrated boosting literature \citep{freund1997decision,freund1999adaptive, telgarsky2013margins, liang2022PreciseHighdimensional} is precisely harnessing the duality between a linear predictive model (aggregating $p$ weak learners) indexed by $\theta \in \R^{p}$ and a finitely supported data distribution (with cardinality $n$) parametrized by a weight on the probability simplex $\mu \in \Delta_n$. There, rather than adversarially perturbing data using the Wasserstein metric, the probability weight vector $\mu$---and its induced joint distribution $\pi_{\mu} = \sum_{i=1}^n \mu_i \delta_{(x_i, y_i)}$---is perturbed as in \eqref{eqn:perturb-wass} under the Kullback-Leibler divergence. A crucial difference between Wasserstein and Kullback-Leibler is that in the latter case, only the weights are allowed to vary but not the domain. Another analogy is regarding the equilibrium concept: when the data set is linearly separable, the equilibrium concept for boosting is the max-margin solution; for our problem, the equilibrium concept is the Bayes optimal solution. The game perspective is also instrumental to the generative modeling and adversarial learning literature \citep{goodfellow2020generative, dziugaite2015training, daskalakis2017training, liang2021HowWell, liang2019interaction, mokhtari2020unified}, where the duality between the probability distribution given by the generative model and discriminative function is leveraged.

\paragraph{Best response and information sets}
Given a covariate distribution $\mu^{(0)}$ whose support is on a linear subspace and does not span the full space $\supp(\mu^{(0)}) \subset X =  \ell^2_{N}$ (so that extrapolation is meaningful), the best response model $f_{\theta^{(0)}} \in \cF$ solves the following risk minimization associated with measure $\mu^{(0)}$ 
\begin{align*}
	\theta^{(0)}  \in \mathrm{BR}(\mu^{(0)}) := \argmin_{\theta\in \ell^2_{\mathbb{N}}}~ \cU(\theta, \mu^{(0)}) \;.
\end{align*}
In both the Gaussian and Bernoulli conditional models, the best response model $\theta^{(0)}$ associated with the measure $\mu^{(0)}$ takes the form
\begin{align*}
	 \theta^{(0)} \in \mathrm{BR}(\mu^{(0)}) =  \left\{ \Pi_{\supp(\mu^{(0)})} \theta^\star + \Pi_{\supp(\mu^{(0)})}^\perp \xi~|~ \forall \xi \in \ell^2_{\mathbb{N}} \right\}\;,
\end{align*}
namely, projected to the linear subspace spanned by $\supp(\mu^{(0)}) \subset X$, the perceived best response model $\theta^{(0)}$ collides the Bayes optimal model $\Pi_{\supp(\mu^{(0)})} \theta^\star$, while on the orthogonal domain $\Pi_{\supp(\mu^{(0)})}^\perp$ no information is learned. 

The minimum Hilbert space norm solution in the over-identified set $\mathrm{BR}(\mu^{(0)})$ is $ \Pi_{\supp(\mu^{(0)})} \theta^\star$. Clearly, $\theta^{(0)} = \Pi_{\supp(\mu^{(0)})} \theta^\star$ is inconsistent with the Bayes optimal model $\theta^\star$ \citep{shimodaira2000improving}. It is, therefore, natural to consider, for typical adversarial distribution shifts $\mu^{(0)} \rightarrow \mu^{(1)}$, how does the information set $\mathrm{BR}(\mu^{(1)})$ differ from $\mathrm{BR}(\mu^{(0)})$? The information set question is useful to understand whether $\theta^{(1)}$ improves upon $\theta^{(0)}$ in approaching the Bayes optimal model $\theta^\star$. To answer this, we will probe precisely how the $\supp(\mu^{(1)})$ varies from $\supp(\mu^{(0)})$ for natural adversarial covariate shifts: what extrapolation regions $\supp(\mu^{(1)})$ focus on.

\paragraph{Adversarial dynamics}
For the adversarial distribution shifts, we follow the Wasserstein gradient flow setup, given a current model $\theta^{(0)}$, the covariate distribution is perturbed incrementally within a Wasserstein ball in an adversarial way: with a stepsize $\gamma \in \R_+$, initialize $\nu_0:= \mu^{(0)}$, 
\begin{align*}
	\nu_{t+1} := \argmin_{\nu \in \cP(X)} ~ -\cU(\theta^{(0)}, \nu) + \frac{1}{\gamma} W_2^2(\nu, \nu_t) \;, ~ \text{for $t=0,1,\ldots, T$}\;,
\end{align*}
and then set $\mu^{(1)} := \nu_{T+1}$. The continuous analog of the adversarial perturbation is called the Wasserstein gradient flow as $\gamma \rightarrow 0$, where the density $\rho_t$ (associated with $\nu_t$) evolves according to the following PDE
\begin{align}
	\partial_t \rho_t + \nabla \cdot( \rho_t \mathrm{V}) = 0 \;, ~\text{where}~\mathrm{V}: x \mapsto \pdv{x} \cU(\theta^{(0)}, \delta_x) \label{eqn:wasserstein-gradient-flow}  \;.
\end{align}
Conceptually, the adversarial distribution shift is a gradient ascent flow on the Wasserstein space $(\cP(X), W_2)$ defined in \eqref{eqn:wasserstein-gradient-flow} with effective time $\gamma T$. In this paper, we study a discretization of \eqref{eqn:wasserstein-gradient-flow} with stepsize $\gamma$ and iterations $T$, as follows
\begin{align}
	x_{t+1} = x_t + \gamma \cdot \pdv{x} \cU(\theta^{(0)}, \delta_{x}) |_{x = x_t}\;, ~ \text{for $t=0,1,\ldots, T$}\;, ~\text{where}~ x_0 \sim \mu^{{(0)}} \;. \label{eqn:discrete-dynamics}
\end{align}

\section{Main Results}
\subsection{Adversarial Covariate Shifts: Blessings and Curses}
\label{sec:blessing-curse}
Let $\theta^{(0)} \in \ell^2_{\mathbb{N}}$ be the current learning model and $\theta^\star - \theta^{(0)}$ be the remaining signal to be identified. Let $\ell^2_{\mathbb{N}}(1)$ denote the unit norm ball. We now define two unit-norm directions: the blessing direction $\Delta_{\mathrm{b}} \in \R^{\mathbb{N}}$ and the curse direction $\Delta_{\mathrm{c}} \in \R^{\mathbb{N}}$
\begin{align}
	\Delta_{\mathrm{b}} &:= \frac{\theta^\star - \theta^{(0)}}{\| \theta^\star - \theta^{(0)} \|} \in \ell^2_{\mathbb{N}}(1) \;, \label{eqn:bless-direction} \\
	\Delta_{\mathrm{c}} &:= - \frac{\| \theta^{(0)}\|}{\| \theta^\star \|} \cdot \frac{\theta^\star - \theta^{(0)}}{\| \theta^\star - \theta^{(0)} \|} + \frac{\| \theta^\star - \theta^{(0)}\|}{\| \theta^\star \|}  \cdot \frac{\theta^{(0)}}{\|  \theta^{(0)} \|} \in \ell^2_{\mathbb{N}}(1) \;. \label{eqn:curse-direction}
\end{align}

The name blessing comes from the fact that $\Delta_{\mathrm{b}} \parallel \theta^\star - \theta^{(0)}$ as stated in Theorem~\ref{thm:regression-main}, namely, the direction is parallel to the remaining signal direction, the most informative signal direction given the current model $\theta^{(0)}$. 

The name curse arises as $\Delta_{\mathrm{c}} \perp \theta^\star$ under the assumption in Theorem~\ref{thm:classification-main}, that is, the direction is perpendicular to the signal direction. To see why this assumption makes sense, we recall that the best response model $\theta^{(0)}$ associated with the measure $\mu^{(0)}$ takes the form
$\theta^{(0)} \in \mathrm{BR}(\mu^{(0)}) =  \left\{ \Pi_{\supp(\mu^{(0)})} \theta^\star + \Pi_{\supp(\mu^{(0)})}^\perp \xi~|~ \forall \xi \in \ell^2_{\mathbb{N}} \right\}$. The minimum-norm solution for the best response set satisfies the assumption in Theorem~\ref{thm:classification-main}, namely $\theta^{(0)} \perp \theta^\star - \theta^{(0)}$ and therefore $\Delta_{\mathrm{c}} \perp \theta^\star$.

Curiously, we will show that the adversarial learning dynamic collapses to a probability measure along the blessing direction $\Delta_{\mathrm{b}}$ in the regression problem; in sharp contrast, the probability measure induced by the adversarial learning dynamic converges to the curse direction $\Delta_{\mathrm{c}}$ in the classification problem. The formal directional convergence results are stated in Theorems~\ref{thm:regression-main} and \ref{thm:classification-main}. For the flow of the exposition, the primary intuition of the proof is deferred to Section~\ref{sec:intuition}. Appendix~\ref{sec:proofs} collects the detailed proof of all theorems.

We first state the result for the infinite-dimensional regression problem. As a reminder, all the relevant notations in Theorems~\ref{thm:regression-main} and \ref{thm:classification-main} were introduced in Section~\ref{sec:problem-setup}.
\begin{theorem}[Regression: directional convergence]
	\label{thm:regression-main}
	Consider the regression setting where $\ell(y', y) = (y'- y)^2$ and $\by|\bx = x \sim \mathrm{Gaussian}\big(\langle x, 
	\theta^\star\rangle, 1 \big)$.
	Let $x_0 \in \supp(\mu^{(0)})$ that satisfies $\langle x_0, \theta^\star - \theta^{(0)} \rangle \neq 0$. Then the induced adversarial distribution shift dynamic \eqref{eqn:discrete-dynamics} satisfies
	\begin{align}
		\lim_{T \rightarrow \infty}~ \left| \big\langle \tfrac{x_T}{\| x_T \|}, \Delta_{\mathrm{b}} \big\rangle \right|= 1 \;, ~\text{where $\Delta_{\mathrm{b}} \parallel \theta^\star - \theta^{(0)}$ is defined in \eqref{eqn:bless-direction}}  \;.
	\end{align}
	Moreover, the directional convergence is exponential in $T$,
	\begin{align*}
		\left|\big \langle \tfrac{x_T}{\| x_T \|}, \Delta_{\mathrm{b}} \big\rangle \right| \in \left[ 1 - O\big(\tfrac{1}{e^{c T}}\big)~,~ 1 \right] \;,
	\end{align*}
	where $c = 2\log(1+2\gamma \| \theta^\star - \theta^{(0)} \|^2)$.	
\end{theorem}
\begin{remark}
	\rm
	This theorem concerns the case when the current model $\theta^{(0)}$ is imperfect---namely $\| \theta^\star - \theta^{(0)} \| \neq 0$---the only case when distribution shifts that vary $\mu \in \cP(X)$ could impact learning the conditional relationship $\pi^\star_x$ and hence identifying the equilibrium $f^\star_{\mathrm{Bayes}}$ defined in \eqref{eqn:bayes}. The current model could be imperfect due to (a) $\supp(\mu^{(0)}) \subsetneq X$ the covariate does not span the full infinite-dimensional space, or (b) the learner only has finite sample access to the measure $\pi_{\mu^{(0)}} \in \cP(X \times Y)$. The theorem states that the adversarial distribution shift dynamics $\mu^{(0)} \rightarrow \mu^{(1)}$ align all the mass of the covariates along the most informative direction for the next stage of learning: the shifted distribution $\mu^{(1)}$ is asymptotically a measure along a one-dimensional ``blessing'' direction $\Delta_{\mathrm{b}}$, reducing the subsequent learning to a one-dimensional problem. Namely, the adversarial distribution shift asymptotically constructs the \textbf{optimal covariate design} for the next stage of learning: making the current model $\theta^{(0)}$ suffer is revealing the information towards the equilibrium of learning, the Bayes optimal model $\theta^\star$. The impact of the distribution shifts on the next stage learner, in this sequential game perspective, is formally stated in Theorem~\ref{thm:learner-reaction-regression}. The proof is based on power iterations as in principle component analysis.
\end{remark}

Now, we state the result for the infinite-dimensional classification problem, which contrasts sharply with the regression problem. We start by stating the conditions and discuss the assumption before stating the theorem.
\begin{assumption}[Initial condition]
	\label{asmp:init-condition}
	Given a fixed $r>0$, we call that two real numbers $(a_0, b_0)$ satisfy the initial condition, if $a_0+b_0 <0$ and
	\begin{align}
		\frac{e^{a_0+b_0} a_0}{1 - e^{2(a_0+b_0)}} &< 1 \;, \label{eqn:assump1} \\
		\frac{e^{a_0+b_0} a_0}{1+ e^{a_0+b_0}} &\geq \frac{1+\frac{1}{a_0}}{ 1+r + \frac{1}{a_0}} \;,\label{eqn:assump2}
	\end{align}
	with $a_0 >c$ for some large enough constant $c>0$.
\end{assumption}
\begin{remark}
	\rm
	When the initialization $a_0 > c$ with $c$ not too small, the assumption holds for a range of $b_0$: \eqref{eqn:assump1} and \eqref{eqn:assump2} are equivalent to
	{\small
	\begin{align*}
	   a_0	\frac{1+\sqrt{1+4 a_0^{-2}}}{2} < e^{-(a_0+b_0)} < a_0 \frac{1+r+a_0^{-1}}{1+a_0^{-1}} - 1 \;.
	\end{align*}
	}This is nonempty whenever $(1+ \frac{r}{1+a_0^{-1}} - a_0^{-1})^2 -1 -4a_0^{-2} > 0$, which is true for $a_0$ not too small.
\end{remark}

\begin{theorem}[Classification: directional convergence]
	\label{thm:classification-main}
	Consider the classification setting where $\ell(y', y) = -y' y + \log(1+e^{y'})$ and $\by|\bx = x \sim \mathrm{Bernoulli}\big( \sigma(\langle x, \theta^\star\rangle) \big)$. Assume $\theta^{(0)} \perp \theta^\star - \theta^{(0)}$ and define $r:= \frac{\| \theta^\star - \theta^{(0)} \|^2}{\| \theta^{(0)} \|^2} >0$. Consider fixed, small step-size in the adversarial distribution shift dynamic \eqref{eqn:discrete-dynamics} that satisfies $\eta := \gamma \| \theta^{(0)} \|^2 < \tfrac{1}{2(2+r)}$. Let $x_0 \in \supp(\mu^{(0)})$ and assume there exists a finite $t_0 \in \mathbb{N}$ with $(a_0, b_0) := \big( \langle x_{t_0}, \theta^{(0)} \rangle, \langle x_{t_0}, \theta^\star - \theta^{(0)} \rangle \big)$ satisfying Assumption~\ref{asmp:init-condition} for constants $r$ defined above and $c>1$. 
	Then 
	\begin{align}
		\lim_{T \rightarrow \infty}~  \left| \big\langle \tfrac{x_T}{\| x_T \|}, \Delta_{\mathrm{c}} \big\rangle  \right| = 1 \;, ~\text{where $\Delta_{\mathrm{c}}  \perp \theta^\star$ is defined in \eqref{eqn:curse-direction}}\;.
	\end{align}
	Moreover, the directional convergence is quadratic in $T/\log(T)$,
	\begin{align*}
		\left|\big\langle \tfrac{x_T}{\| x_T \|}, \Delta_{\mathrm{c}} \big\rangle   \right| \in \left[  1 - O\big( \tfrac{\log^2(T)}{T^2} \big)~,~ 1\right] \;.
	\end{align*}
\end{theorem}
\begin{remark}
	\rm
	This theorem concerns also the case when the current model $\theta^{(0)}$ is useful yet imperfect---namely $\| \theta^\star - \theta^{(0)} \| \neq 0$, and $\| \theta^\star\| \neq 0$. In particular, this theorem studies when $\supp(\mu^{(0)}) \subsetneq X$. As stated before, the best response model $\theta^{(0)}$ associated with the measure $\mu^{(0)}$ takes the form
$\theta^{(0)} \in \mathrm{BR}(\mu^{(0)}) =  \left\{ \Pi_{\supp(\mu^{(0)})} \theta^\star + \Pi_{\supp(\mu^{(0)})}^\perp \xi~|~ \forall \xi \in \ell^2_{\mathbb{N}} \right\}$.
The minimum-norm solution for the best response set satisfies the assumption in the theorem, namely $\theta^{(0)} \perp \theta^\star - \theta^{(0)}$. 
The theorem is also about directional convergence: the adversarial distribution shift dynamics $\mu^{(0)} \rightarrow \mu^{(1)}$ asymptotically align all the mass of the covariates along a one-dimensional, ``curse'' direction $\Delta_{\mathrm{c}} \perp \theta^\star$, orthogonal to the Bayes optimal model. This directional alignment will reduce the subsequent learning to a one-dimensional problem that is the hardest, namely, the $(x, y) \sim \pi_{\mu^{(1)}}$ where $y$ is a Bernoulli coin-flip that is independent of $x$! Note the adversarial distribution shift (under the logistic loss) asymptotically constructs the \textbf{hardest covariate design} under the 0-1 loss, since the Bernoulli coin-flip is impossible to predict for the next stage of learning. Qualitatively, this contrasts sharply with the phenomenon in the regression setting. The adversarial distribution shift constructs a difficult covariate design, trapping the next stage of learning. Namely, making the current model $\theta^{(0)}$ suffer is constructing the hardest experimental design for identifying the Bayes optimal model $\theta^\star$. The impact of the distribution shifts on the next stage learner, in this sequential game perspective, is formally stated in Theorem~\ref{thm:learner-reaction-classification}. Quantitatively, the directional convergence to $\Delta_{\mathrm{c}}$ in the classification setting is quadratic in $T$, much slower than the directional convergence to $\Delta_{\mathrm{b}}$ in the regression setting, exponential in $T$. We invite the readers to Section~\ref{sec:numerics} for a preliminary numerical experiment illustrating the sharp contrast of Theorems~\ref{thm:regression-main} and \ref{thm:classification-main}, visualized in Figure~\ref{fig:regression} and Figure~\ref{fig:classification}. We find the contrast in the speed of directional convergence interesting in its own right.

Tracking down the exact behavior of the distribution shift dynamic (non-convex and non-linear) and establishing a sharp directional convergence rate are this paper's main technical innovations and difficulties. To overcome these challenges, we carefully construct two bounding envelopes refined recursively to characterize the dynamics analytically. Section~\ref{sec:intuition} elaborates on the main steps of the technical proof and key ideas.
% Discussions regarding the initial condition on $x_0$ will also be found in Section~\ref{sec:intuition}.
\end{remark}
Before studying the subsequent impact on the learner, we briefly discuss why a sharp dichotomy exists between regression and classification. A crucial factor is the loss function. One intuition that helps shed light on the distinction between classification and regression can be seen in the following special case. Consider when the learner $\theta^{(0)}$ has already learned the Bayes optimal function $\theta^\star$: (a) for the regression problem, the covariate distribution will not shift, since the gradient flow for covariate distribution has velocity zero, as a result of the quadratic loss function; (b) in contrast, for the classification problem, the covariate distribution will further shift to decrease the margin so to decrease the signal-to-noise ratio, as a result of the logistic loss function. For the general case when $ \theta^\star - \theta^{(0)} \neq 0$, the explicit directional convergence phenomenon is curious and technically challenging to derive, which is the main focus of the current paper.

\subsection{Impact on the Learner: Sequential Game Perspective}
\label{sec:impact-on-learner}

In this section, we investigate the impact of the covariate distribution shift on the next stage learner's gradient descent dynamic. The goal here is to demonstrate that the directional convergence results established in Theorems~\ref{thm:regression-main} and \ref{thm:classification-main} can translate to a direct impact on the learner's subsequent learning towards the Bayes optimal model $f^\star_{\mathrm{Bayes}}(x)  = \langle x, \theta^\star \rangle$, the equilibrium. The impact is a dichotomy that either comes as a blessing or a curse.

The sequential game between model $f_{\theta}$ and covariate distribution $\mu$ evolves according to the following protocol. Here we only focus on one round, namely from stage $t=0$ to $t=1$. 
\begin{enumerate}
	\item Adversarial covariate shift: the covariate distribution shifts from $\widehat{\mu}^{(0)} \rightarrow \widehat{\mu}^{(1)}$ given the previous model $\theta^{(0)}$. Here $\widehat{\mu}^{(0)}  = \frac{1}{n} \sum_{i=1}^n \delta_{x^{i}_0/\| x^{i}_0 \|}$ where $x^{i}_0 \sim \mu^{(0)}$, and $\widehat{\mu}^{(1)} := \frac{1}{n} \sum_{i=1}^n \delta_{x^{i}_{T}/\| x^{i}_{T} \|}$, where $x^{i}_{T}$ evolves according to \eqref{eqn:discrete-dynamics} after $T$ steps.
	\item Learner's subsequent action: the learner performs a one-step improvement using gradient descent given the shifted distribution $\widehat{\mu}^{(1)}$. Draw $n$-i.i.d. samples $(x^i, y^i)\sim \pi_{\widehat{\mu}^{(1)}}$ defined in \eqref{eqn:disintegration}, and update
	\begin{align}
		\label{eqn:learner-gradient-descent}
		\theta^{(1)} = \theta^{(0)} - \eta \cdot \frac{1}{n} \sum_{i=1}^n \pdv{\theta} \ell(\langle x^i, \theta \rangle, y^i)|_{\theta = \theta^{(0)}} \;.
	\end{align}
	Note that $\theta^{(1)}$ implicitly depends on $(n, T, \eta)$ so we subsequently denote as $\theta^{(1)}_{n, T, \eta}$.
\end{enumerate}
The curious reader may wonder about the renormalization such that $\widehat{\mu}^{(1)}$ is a probability measure on the unit sphere. Note that this is for convenience of analysis and does not change the qualitative phenomenon. The assumption is currently used to invoke the Hanson--Wright inequality in the proofs of Theorems~\ref{thm:learner-reaction-regression} and \ref{thm:learner-reaction-classification}.
\begin{theorem}[Regression: blessing to the learner]
	\label{thm:learner-reaction-regression}
	Consider the same setting as in Theorem~\ref{thm:regression-main}. For all $\theta^{(0)}$ such that $\| \theta^\star - \theta^{(0)} \| \neq 0$, the learner's one-step reaction to the distribution shift as in \eqref{eqn:learner-gradient-descent} with $\eta = 1/2$ satisfies
	\begin{align}
		\lim_{n \rightarrow \infty} \lim_{T\rightarrow \infty}~ \| \theta^\star - \theta^{(1)}_{n, T, \eta} \|  = 0 ~~\text{a.s.} \;. \label{eqn:regression-impact}
	\end{align}
\end{theorem}

\begin{theorem}[Classification: curse to the learner]
	\label{thm:learner-reaction-classification}
	Consider the same setting as in Theorem~\ref{thm:classification-main}. For all $\theta^{(0)}$ such that $\langle \theta^\star - \theta^{(0)}, \theta^\star \rangle \neq 0$, the learner's one-step reaction to the distribution shift as in \eqref{eqn:learner-gradient-descent} with any fixed $\eta>0$ satisfies
	\begin{align*}
		\lim_{n \rightarrow \infty} \lim_{T\rightarrow \infty}~ \frac{\langle \theta^\star - \theta^{(1)}_{n, T, \eta}, \theta^\star \rangle}{\langle \theta^\star - \theta^{(0)}, \theta^\star \rangle} = 1 \;.
	\end{align*}
	Moreover,
	\begin{align}
		\liminf_{n \rightarrow \infty} \lim_{T\rightarrow \infty}~ \| \theta^\star - \theta^{(1)}_{n, T, \eta} \|  > 0 \;. \label{eqn:classification-impact}
	\end{align}
\end{theorem}
\begin{remark}
	\rm
	The result in Theorem~\ref{thm:learner-reaction-classification} holds valid for any fixed number of gradient descent steps for the learner
	\begin{align*}
		\theta^{(t+1)} = \theta^{(t)} - \eta \cdot \frac{1}{n} \sum_{i=1}^n \pdv{\theta} \ell(\langle x^i, \theta \rangle, y^i)|_{\theta = \theta^{(t)}} \;.
	\end{align*}
	Therefore, the learner gets stuck with the adversarial distribution $\widehat{\mu}^{(1)}$ in making progress towards $\theta^\star$. Comparing Theorems~\ref{thm:learner-reaction-regression} and \ref{thm:learner-reaction-classification}, we see that the adversarial distribution shifts in the regression setting make the learner's one-step subsequent move optimal! \eqref{eqn:regression-impact} shows that one-step improvement using gradient descent dynamic as in \eqref{eqn:learner-gradient-descent} will reach the Bayes optimal model, also the equilibrium to the minimax game as in \eqref{eqn:minimax}. On the contrary, in the classification setting, \eqref{eqn:classification-impact} shows that subsequent learner's move using gradient descent dynamic (regardless of the number of steps) will be trapped with no improvement, preventing the learner from reaching the Bayes optimal model.
\end{remark}

\subsection{Numerical Illustration}
\label{sec:numerics}
In this section, we provide two simple numerical simulations, one for regression and one for classification, to contrast the sharp differences of the directional convergence established in Theorems~\ref{thm:regression-main} and \ref{thm:classification-main}. 

\paragraph{Experiment setup} Both simulations are based on a high-dimensional setting where $X = \mathbb{R}^{200}$ and $\supp(\mu^{(0)}) = \mathbb{R}^{100}$ a randomly drawn subspace from the Haar measure (a uniformly drawn subspace of dimension 100). Specify a Bayes optimal model $\theta^\star = [1, 1/2, \ldots, 1/i,\ldots, 1/200]^\top$, denoted by the the {\color{orange} Yellow $\star$} in the figures. The best response model restricted to the subspace $\supp(\mu^{(0)})$, $\{\tfrac{\theta^{(0)}}{\| \theta^{(0)} \|}$ direction is noted by {\color{ForestGreen} Green $\blacktriangle$}, and the remaining signal direction $\tfrac{\theta^\star - \theta^{(0)}}{\|\theta^\star - \theta^{(0)} \|}$  is denoted by {\color{red} Red $\times$}. The probability distribution of the covariates is shown by the {\color{MidnightBlue} Blue $\bullet$}. For the experiment, we consider the case when $\theta^{(0)} \perp \theta^\star - \theta^{(0)}$. To visualize the adversarial distribution shift on a two-dimensional plane, we project all points to the two-dimensional subspace with horizontal and vertical axes $\{\tfrac{\theta^{(0)}}{\| \theta^{(0)} \|}, \tfrac{\theta^\star - \theta^{(0)}}{\|\theta^\star - \theta^{(0)} \|}\}$. In each simulation, the adversarial distribution shift dynamic is visualized by the motion of the {\color{MidnightBlue} Blue $\bullet$} data clouds. Since we focus on directional convergence, every data point is visualized by its direction, normalized to the unit ball in $\mathbb{R}^{200}$. Concretely, for each draw $x_0 \sim \mu^{(0)}$, we plot at each timestamp $\tfrac{x_t}{\| x_t \|}$, projected to the plane. 
\begin{figure}[ht]
\centering
\includegraphics[width=0.8\linewidth]{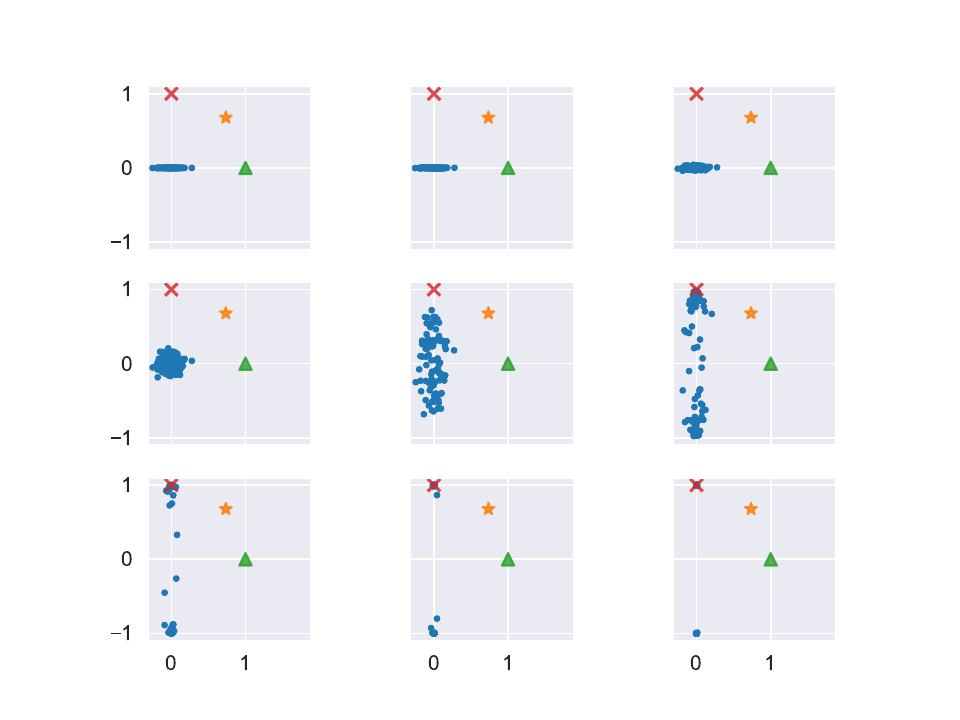}
\caption{\small Regression setting, directional convergence. From left to right, top to bottom, we plot the directional information at timestamp $t=0, 5, 10, \ldots, 40$, once every $5$ iterations.}
\label{fig:regression}
\end{figure}
\paragraph{Directional convergence: regression}
Consider \eqref{eqn:utility} with $\by|\bx = x \sim \mathrm{Gaussian}\big(\langle x, 
	\theta^\star\rangle, 1 \big)$ and $\ell(y', y) = (y'- y)^2$. 
The adversarial distribution shift evolves according to \eqref{eqn:discrete-dynamics}. Here the blessing direction $\Delta_{\mathrm{b}}$ defined in \eqref{eqn:bless-direction} is precisely marked by {\color{red} Red $\times$}. As seen in Figure~\ref{fig:regression}, the {\color{MidnightBlue} Blue $\bullet$} data clouds collapsed to a align perfectly to $\Delta_{\mathrm{b}}$, rapidly. We emphasize that since the simulation is done in high dimensions, directional convergence $|\langle \tfrac{x_t}{\| x_t \|}, \Delta_{\mathrm{b}}  \rangle| =1$ is only true when the {\color{MidnightBlue} Blue $\bullet$} perfectly lands on $\{ \pm \Delta_{\mathrm{b}} \}$, shown in $t=40$ (the bottom right subfigure); just aligning to a direction in the two-dimensional domain does not imply directional convergence in $\R^{200}$.
\begin{figure}[ht]
\centering
\includegraphics[width=0.8\linewidth]{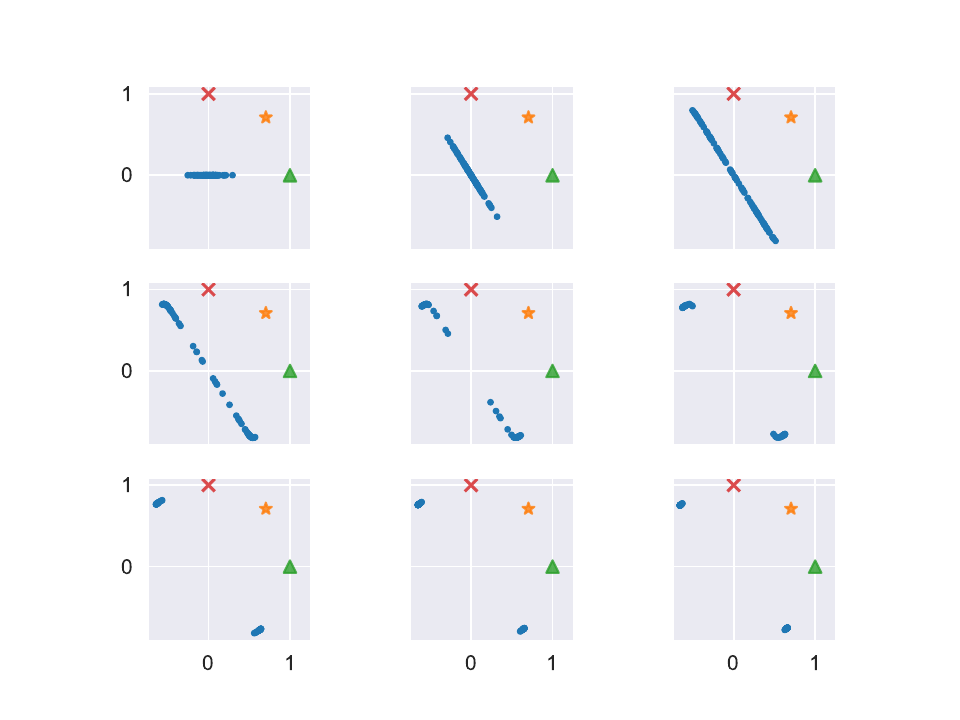}
\caption{\small Classification setting, directional convergence. From left to right, top to bottom, we plot the directional information at timestamp $t=0, 25, 50, \ldots, 200$, once every $25$ iterations.}
\label{fig:classification}
\end{figure}
\paragraph{Directional convergence: classification}
Consider \eqref{eqn:utility} with  $\by|\bx = x \sim \mathrm{Bernoulli}\big( \sigma(\langle x, \theta^\star\rangle) \big)$ and $\ell(y', y) = -y' y + \log(1+e^{y'})$.
The adversarial distribution shift evolves according to \eqref{eqn:discrete-dynamics}. Here the curse direction $\Delta_{\mathrm{c}}$ defined in \eqref{eqn:curse-direction} is perpendicular to {\color{orange} Yellow $\star$}. Seen in Figure~\ref{fig:classification}, the {\color{MidnightBlue} Blue $\bullet$} data clouds eventually land on the directions $\{ \pm \Delta_{\mathrm{c}} \}$. Compared to Figure~\ref{fig:regression}, the direction $\{ \pm \Delta_{\mathrm{c}} \}$ is different from the regression case $\{ \pm \Delta_{\mathrm{b}} \}$, and the convergence is much slower ($T^{-2}\log^2(T)$ vs. $\exp(-T)$), as proved by Theorems~\ref{thm:learner-reaction-regression} and \ref{thm:learner-reaction-classification}. Again we emphasize that alignment to a direction in the two-dimensional domain does not imply directional convergence in $\R^{200}$: $t=50$ (the top right subfigure) $|\langle \tfrac{x_t}{\| x_t \|}, \Delta_{\mathrm{c}}  \rangle| \neq 1$; only when $t = 175$ (the bottom middle subfigure), we have roughly directional convergence $|\langle \tfrac{x_t}{\| x_t \|}, \Delta_{\mathrm{c}}  \rangle| =1$.

\subsection{Intuition of the Proof}
\label{sec:intuition}

In this section, we elaborate on the intuition and technical innovation behind the proof of Theorem~\ref{thm:classification-main}. Let $\sigma(z) = 1/(1+e^{-z})$ be the sigmoid function and $\sigma'(z)=\sigma(z)(1-\sigma(z))$. The dynamic of the distribution shift is non-convex and non-linear, following the equation
\begin{align*}
	x_{t+1} =x_{t} + \gamma \cdot \left[ -\sigma(\langle x_{t}, \theta^\star \rangle) \big(1 - \sigma(\langle x_{t}, \theta^\star \rangle)\big) \langle x_t, \theta^{(0)} \rangle \cdot \theta^\star + \big( \sigma(\langle x_{t},  \theta^{(0)} \rangle) - \sigma(\langle x_{t}, \theta^\star \rangle)\big) \cdot \theta^{(0)} \right].
\end{align*}
Define $\eta := \gamma \| \theta^{(0)} \|^2 >0$ and $r:= \frac{\| \theta^\star - \theta^{(0)} \|^2}{\| \theta^{(0)} \|^2}$. We focus on two summary statistics to keep track of the directional convergence. For all $t\geq 0$, define a sequence of real values $a_t, b_t \in \R$
	\begin{align*}
		a_{t} := \langle x_t, \theta^{(0)} \rangle \;, ~~b_{t} := \langle x_t, \theta^\star - \theta^{(0)}\rangle \;.
	\end{align*}
The non-linear evolution of the summary statistics is thus defined,
	\begin{align*}
		a_{t+1} &= a_{t} - \eta \cdot \sigma'(a_{t}+b_{t}) a_{t} + \eta \cdot \big( \sigma(a_{t}) - \sigma(a_{t}+b_{t})\big) \;, \\
		b_{t+1} &= b_{t} - \eta \cdot r\sigma'(a_{t}+b_{t}) a_{t} \;.
	\end{align*}
Recall that $b_0 = \langle x_0, \theta^\star - \theta^{(0)}\rangle = \langle x_0, \theta^\star - \Pi_{\supp(\mu^{(0)})}  \theta^\star \rangle =0$ since  $\theta^{(0)} \in \mathrm{BR}(\mu^{(0)})$ and $x_0 \in \supp(\mu^{(0)})$. Without loss of generality, we consider $a_0 = \langle x_0, \theta^{(0)}\rangle >0$. The directional convergence now hinges on studying $\{ a_t, b_t \}_{t\geq 0}$, done in Lemma~\ref{lem:key-limit}.

The proof builds upon the following ``rough'' observation of $\{ a_t, b_t \}_{t\geq 0}$: after a finite time $t_0$, a key quantity ($L$ for Lyapunov)
\begin{align}
	\label{eqn:Lyapunov}
	L_t:= \frac{\sigma'(a_t+b_t) a_t}{\sigma(a_t) - \sigma(a_t+b_t)} < 1
\end{align}
will cross below threshold $1$ and deviate away from the threshold $1$ for $t\geq t_0$. However, perhaps surprisingly, one can show even when $t\rightarrow \infty$, the quantity never cross below a threshold
\begin{align*}
	L_t \geq \frac{1}{1+r}, ~\forall t\geq t_0 \;.
\end{align*}
Namely, the threshold $\frac{1}{1+r}$ is a stable fixed point for the quantity $L_t$. The quantity $L_t$ regulates the monotonicity of the updates $a_t, b_t, a_t+b_t$ and determines the order of magnitude for each term, see Lemma~\ref{lem:nonlinear-recur}.

The above intuition is educative but hard to directly operate upon over iterations, due to the non-linear form of \eqref{eqn:Lyapunov} and the nonlinear recursions of $\{a_t, b_t\}$. Instead of directly working with $L_t$, we build two envelopes inspired by $L_t$ that are easier to control during recursions, done in Lemma~\ref{lem:recur-estimate}. We define
\begin{align*}
	L^{\rm env\mathchar`-U}_t := \frac{e^{a_t+b_t} a_t}{1 - e^{2(a_t+b_t)}}, ~\text{and}~ L^{\rm env\mathchar`-L}_t := \frac{e^{a_t+b_t} a_t}{1 + e^{a_t+b_t}} \;.
\end{align*}
We show that these two envelopes are related to $L_t$ in the following sense (Lemma~\ref{lem:helper-function}), but not sandwiching it (we only have $L^{\rm env\mathchar`-L}_t \leq \min\{ L_t, L^{\rm env\mathchar`-U}_t \}$) 
\begin{align*}
	L^{\rm env\mathchar`-U}_t < 1 \implies L_t < 1, ~\text{and}~ L^{\rm env\mathchar`-L}_t > \frac{1}{1+r} \implies L_t > \frac{1}{1+r} \;.
\end{align*}
It turns out that the envelopes intervene cleanly with the non-linear recursions for $t \rightarrow t+1$. The crux of the argument lies in a strengthened version of the ``rough observation'' outlined in the previous paragraph, specifically done in Lemma~\ref{lem:recur-estimate}. 
On the one hand, if the lower envelope function $L^{\rm env\mathchar`-L}_t > \frac{1+a_t^{-1}}{1+r + a_t^{-1}} \in [\frac{1}{1+r}, 1]$, then the upper envelope function decreases in the recursion, $L^{\rm env\mathchar`-U}_{t+1} < L^{\rm env\mathchar`-U}_{t} < 1$; On the other hand, the lower envelope function cannot decrease too much, in the following sense
\begin{align*}
	L^{\rm env\mathchar`-L}_t > \frac{1+a_t^{-1}}{1+r + a_t^{-1}} \implies L^{\rm env\mathchar`-L}_{t+1} > \frac{1+a_{t+1}^{-1}}{1+r + a_{t+1}^{-1}} > \frac{1}{1+r} \;.
\end{align*}
The two envelope functions also ensure the monotonicity of $a_t+b_t \downarrow -\infty$ and $a_t \uparrow +\infty$. To sum up, we can show that the $L_t$ has $\frac{1}{1+r}$ as the stable fixed point. The analytical characterization of $L_t$ ensures an explicit rate on the directional convergence. Finally, the Assumption~\ref{asmp:init-condition} is a mild condition requiring the dynamic of $L_t$ to reach below $1$. All the detailed proof can be found in Appendix~\ref{sec:proofs}.

\section{Discussion and Future Work}

This paper studied covariate shifts from game-theoretic and dynamic viewpoints, with the underlying Bayes optimal model being invariant. In particular, we show that under the Wasserstein gradient flow, the distribution of covariates will converge in directions in both regression and classification. However, the result presents as a dichotomy: a blessing in regression, the adversarial covariate shifts in an exponential rate to an optimal experimental design for rapid subsequent learning; a curse in classification, the adversarial covariate shifts in a subquadratic rate fast to the hardest experimental design trapping subsequent learning. We view the work as a starting point for unveiling new insights for adversarial learning and covariate shift: it cautions the researchers to be aware that robust optimization is brittle depending on the learning losses, model complexity, and early stopping. In particular, following potential directions are left as future work.

\paragraph{Discriminative vs. Generative}
This paper considers discriminative models for the joint distribution, where the conditional distribution $P(Y|X)$ stays invariant, yet the covariate distribution $P(X)$ can shift. The main reason is to define the Bayes optimal model as the equilibrium, invariant regardless of the covariate distribution $P(X)$. We study at the population level to simplify the main results and analysis. Another line of literature on adversarial examples, in the classification setting only, considers a generative model, where $P(Y)$ stays untouched, yet $P(X|Y)$ are allowed to shift. By simple Bayes rule, the Bayes optimal model $P(Y|X)$ will consequently vary, making the concept a moving target. It is still to be determined if the notion of equilibrium or invariance exists in the generative setting. We leave it as a future direction for investigation.

\paragraph{Iterative Game Updates}
This paper only considers running the subsequent learning after the covariate shift reaches stationarity in direction. Generally, one may envision the game between the learner and nature by iteratively running gradient descent and ascent dynamics. The non-asymptotic analysis for iterative game updates, and the trade-offs on step sizes between the learner and the covariate shift will require future work. The analysis of iterative game updates could also benefit the understanding of generative models like generative adversarial networks.

\paragraph{Complex Models}
This paper shows curious insights into studying infinite-dimensional linear models with square and logistic loss. What will happen for other nonlinear models, such as neural networks? Extensions to nonlinear models require highly technical work and could lead to new insights on covariate shifts.

\acks{Liang acknowledges the generous support from the NSF CAREER Grant (DMS-2042473) and the William Ladany Faculty Fellowship from the University of Chicago Booth School of Business.}
%%%%%%%%%%%%%%%%%%%%%%%%%%%%%%%%%%%%%%%%%%%%%%
%% Reference                                %%
%%%%%%%%%%%%%%%%%%%%%%%%%%%%%%%%%%%%%%%%%%%%%%

\vskip 0.2in
\bibliography{ref-rev.bib}

\newpage
\appendix

\section{Proofs}
\label{sec:proofs}
\subsection{Proofs in Section~\ref{sec:blessing-curse}}
\begin{proof}[Proof of Theorem~\ref{thm:regression-main}]
	In the regression setting, the utility evaluated at delta measure $\delta_x$ takes the form
	\begin{align*}
		\cU(\theta, \delta_{x}) = \int_Y \ell(f_{\theta}(x), y)\dd{\pi^\star_{x}(y)} = \langle x, \theta^\star - \theta \rangle^2 + 1 \;.
	\end{align*}
	For each particle $x_0\sim \mu^{(0)}$, the adversarial distribution shift updates following the iteration
	\begin{align*}
		x_{t+1} &= x_{t} + \gamma \cdot \pdv{x} \cU(\theta^{(0)}, \delta_{x_t}) \;, \\
		&= \big[ I + 2\gamma (\theta^\star - \theta^{(0)}) (\theta^\star - \theta^{(0)})^\top \big] x_t \;,\\
		&= ( I + \tilde{\gamma} \Delta_{\mathrm{b}} \Delta_{\mathrm{b}}^\top ) x_t \;,
	\end{align*}
	where $\tilde{\gamma} := 2\gamma \| \theta^\star - \theta^{(0)} \|^2$. Therefore
	\begin{align*}
		x_{T}=  (1+\tilde{\gamma})^T \langle x_{0}, \Delta_{\mathrm{b}} \rangle \Delta_{\mathrm{b}} + (I - \Delta_{\mathrm{b}} \Delta_{\mathrm{b}}^\top) x_{0} \;.
	\end{align*}
	It is then clear to see that
	\begin{align*}
		\langle x_T, \Delta_{\mathrm{b}} \rangle &= (1+\tilde{\gamma})^T \langle x_{0}, \Delta_{\mathrm{b}} \rangle \;,\\
		\| x_T \| &= \left[ (1+\tilde{\gamma})^{2T} \langle x_{0}, \Delta_{\mathrm{b}} \rangle^2 +  \langle x_{0}, (I - \Delta_{\mathrm{b}} \Delta_{\mathrm{b}}^\top) x_{0} \rangle \right]^{1/2} \;,
	\end{align*}
	and thus, we can conclude the proof by noting
	\begin{align*}
		\left| \frac{\langle x_T, \Delta_{\mathrm{b}} \rangle}{\| x_T \|} \right|  = \frac{1}{\left[ 1 +  \tfrac{\langle x_{0}, (I - \Delta_{\mathrm{b}} \Delta_{\mathrm{b}}^\top) x_{0} \rangle }{(1+\tilde{\gamma})^{2T} \langle x_{0}, \Delta_{\mathrm{b}} \rangle^2}\right]^{1/2}} \;.
	\end{align*}
\end{proof}

\begin{proof}[Proof of Theorem~\ref{thm:classification-main}]
	Let $\sigma(z) = 1/(1+e^{-z})$ be the sigmoid function and $\sigma'(z)=\sigma(z)(1-\sigma(z))$. 
	In the classification setting, 
	\begin{align*}
		\cU(\theta, \delta_{x}) = \int_Y \ell(f_{\theta}(x), y)\dd{\pi^\star_{x}(y)} = \frac{1}{1+e^{-\langle x, \theta^\star\rangle}}\log(1+ e^{-\langle x, \theta\rangle}) + \frac{1}{1+e^{\langle x, \theta^\star\rangle}}\log(1+ e^{\langle x, \theta\rangle}) \;.
	\end{align*}
	For each particle $x_0\sim \mu^{(0)}$, the adversarial distribution shift reads
	\begin{align*}
		x_{t+1} &= x_{t} + \gamma \cdot \pdv{x} \cU(\theta^{(0)}, \delta_{x_t}) \;, \\
		&=x_{t} + \gamma \cdot \left[ -\sigma(\langle x_{t}, \theta^\star \rangle) \big(1 - \sigma(\langle x_{t}, \theta^\star \rangle)\big) \langle x_t, \theta^{(0)} \rangle \cdot \theta^\star + \big( \sigma(\langle x_{t},  \theta^{(0)} \rangle) - \sigma(\langle x_{t}, \theta^\star \rangle)\big) \cdot \theta^{(0)} \right] \;.
	\end{align*}
	 For all $t\geq 0$, define a sequence of real values $(a_t, b_t) \in \R^2$
	\begin{align*}
		a_{t} := \langle x_t, \theta^{(0)} \rangle \;, ~~b_{t} := \langle x_t, \theta^\star - \theta^{(0)}\rangle \;.
	\end{align*}
	Let $\eta := \gamma \| \theta^{(0)} \|^2 >0$. 
	Then the recursive relationship on $\{a_t, b_t\}$ induced by the dynamics reads
	\begin{align*}
		a_{t+1} &= a_{t} - \eta \cdot (1+\sqrt{r}q)\sigma'(a_{t}+b_{t}) a_{t} + \eta \cdot \big( \sigma(a_{t}) - \sigma(a_{t}+b_{t})\big) \;, \\
		b_{t+1} &= b_{t} - \eta \cdot (\sqrt{r}q+r)\sigma'(a_{t}+b_{t}) a_{t} + \eta \cdot \sqrt{r}q \big( \sigma(a_{t}) - \sigma(a_{t}+b_{t})\big) \;,
	\end{align*}
	where $q := \left\langle \frac{\theta^\star - \theta^{(0)}}{\| \theta^\star - \theta^{(0)} \|}  , \frac{\theta^{(0)}}{\|  \theta^{(0)} \|} \right\rangle$ and $r:= \frac{\| \theta^\star - \theta^{(0)} \|^2}{\| \theta^{(0)} \|^2}$. 
	To study the directional convergence, we need sharp characterizations of the sequence $\{a_t, b_t\}$ following the above non-linear, non-convex updates. The main technical hurdle is to derive a precise estimate on the following quantity, done in Lemma~\ref{lem:nonlinear-recur}, 
	\begin{align*}
		\left| \frac{a_T+ b_T}{a_T} \right| = O\left(\frac{\log(T)}{T} \right) \;.
	\end{align*}
	To apply Lemma~\ref{lem:nonlinear-recur}, we note the assumption $\theta^{(0)} \perp \theta^\star - \theta^{(0)}$, which in turn simplifies the dynamics with $q = 0$.
	
	Recall that the direction $\Delta_{\mathrm{c}}$ can be written as
	\begin{align*}
		\Delta_{\mathrm{c}} = -\frac{1}{\sqrt{1+r}} \cdot \frac{\theta^\star - \theta^{(0)}}{\| \theta^\star - \theta^{(0)} \|} + \frac{\sqrt{r}}{\sqrt{1+r}}  \cdot \frac{\theta^{(0)}}{\|  \theta^{(0)} \|} \;.
	\end{align*}
	Therefore
	\begin{align*}
		\langle x_T, \Delta_{\mathrm{c}} \rangle &= -\frac{1}{\sqrt{1+r}} \cdot \frac{b_{T}}{\| \theta^\star - \theta^{(0)} \|} + \frac{\sqrt{r}}{\sqrt{1+r}}  \cdot \frac{a_{T}}{\|  \theta^{(0)} \|} \; \\
		\| x_T \| &= \left[ \frac{b^2_{T}}{\| \theta^\star - \theta^{(0)} \|^2} + \frac{a^2_{T}}{\|  \theta^{(0)} \|^2 } +  \big\langle x_0, \Pi^\perp_{\{ \theta^{(0)},\theta^\star - \theta^{(0)}\}}  x_0\big\rangle\right]^{1/2} \;,
	\end{align*}
	and thus we conclude the proof recalling Lemma~\ref{lem:nonlinear-recur}
	\begin{align*}
		 \frac{\langle x_T, \Delta_{\mathrm{c}} \rangle}{\| x_T \|} &= \frac{(1+r) a_T - (a_T +b_T)}{\sqrt{1+r}\sqrt{b_T^2 + r a_T^2 + \| \theta^\star - \theta^{(0)} \|^2  \big\langle x_0, \Pi^\perp_{\{ \theta^{(0)},\theta^\star - \theta^{(0)}\}}  x_0\big\rangle}} \\
		 &= \frac{1 - \frac{1}{1+r} \frac{a_T+b_T}{a_T}}{\sqrt{1-\frac{2}{1+r} \frac{a_T+b_T}{a_T} + \frac{1}{1+r}\big[ \big(  \frac{a_T+b_T}{a_T} \big)^2 + \frac{\langle x_0, \Pi^\perp_{\{ \theta^{(0)},\theta^\star - \theta^{(0)}\}}  x_0\rangle}{a_T^2} \big]}} \;.
	\end{align*}
\end{proof}

\begin{lemma}[Nonlinear recursions]
	\label{lem:nonlinear-recur}
	Let $\sigma(z) = 1/(1+e^{-z})$ be the sigmoid function and $\sigma'(z)=\sigma(z)(1-\sigma(z))$. Define the nonlinear recursion for fixed $r,\eta>0$, with $a_0>0$ and $b_0 = 0$
	\begin{align}
		a_{t+1} &= a_{t} - \eta \cdot \sigma'(a_{t}+b_{t}) a_{t} + \eta \cdot \big( \sigma(a_{t}) - \sigma(a_{t}+b_{t})\big) \;, \label{eqn:a_t} \\
		b_{t+1} &= b_{t} - \eta \cdot r\sigma'(a_{t}+b_{t}) a_{t} \;.\label{eqn:b_t} 
	\end{align}
	Assume there exists some $t_0 \in \mathbb{N}$,  such that $(a_{t_0}, b_{t_0})$ satisfy Assumption~\ref{asmp:init-condition} with $r>0, c >1$. Then as $T \rightarrow \infty$, $a_T + b_T \rightarrow -\infty$ and $a_T \rightarrow +\infty$ and
	\begin{align*}
		 |a_T+b_T| &= O(\log(T)) \;, \\
		 a_T &= \Theta(T) \;.
	\end{align*}
\end{lemma}

\begin{proof}[Proof of Lemma~\ref{lem:nonlinear-recur}]
	$a_{t_0}$ and $b_{t_0}$ satisfy the conditions in Lemma~\ref{lem:key-limit}, therefore we know $a_{t+1}>a_{t}>0$, and $a_{t+1}+b_{t+1} < a_{t} + b_{t} < 0$ for $t \geq t_0$. Given the monotonicity, the proof proceeds in the following steps. First, note that
\begin{align}
		ra_{t+1} - b_{t+1} &= r a_t- b_{t} +  \eta \cdot r \big( \sigma(a_t) - \sigma(a_t+b_t) \big) \label{eqn:exact-key-1} \;,\\
		a_{t+1}+b_{t+1} &= \big(1 - \eta \sigma'(a_t+b_t)\big) (a_t +b_t) - \eta \sigma'(a_t+b_t) (ra_t - b_t) + \eta \big( \sigma(a_t) - \sigma(a_t+b_t) \big) \label{eqn:exact-key-2} \;.
	\end{align} 
	Below we use the Bachmann–Landau notation: we say two sequences of positive real numbers $x_t = O(z_t)$ iff $\limsup_{t\rightarrow \infty} x_t/z_t < \infty$, $x_t = \Omega(z_t)$ iff $\liminf_{t\rightarrow \infty} x_t/z_t > 0$, and $x_t = \Theta(z_t)$ iff $x_t = O(z_t)$ and $x_t = \Omega(z_t)$.
	\begin{enumerate}
		\item Observe that $a_t > 0$, $b_t \leq 0$ and $b_{t+1} < b_{t}$ monotonic decreasing. The proof is straightforward from induction observing the form of \eqref{eqn:a_t}-\eqref{eqn:b_t}.

		\item Claim $ra_t - b_t \rightarrow +\infty$. Proof by contradiction. If the monotonic sequence $ra_t - b_t$ is uniformly upper bounded by $M$. Then $ \sigma(a_t) - \sigma(a_t+b_1) \leq  \sigma(a_t) - \sigma(a_t+b_t)  = o(1)$, which implies $a_t \rightarrow \infty$, and therefore contradicts $r a_t \leq r a_t - b_t \leq M$.

		\item Claim $b_t \rightarrow -\infty$. Proof by contradiction. If $b_t$ is bounded below by $-M$, then $\eta \cdot r \sigma'(a_t+b_t) a_t \rightarrow 0$. Note $a_t \geq a_{t_0} >0$, then $\sigma'(a_t+b_t) \rightarrow 0$. Recall that $a_t+b_t \leq a_{t_0} + b_{t_0} <0$ and is monotonic decreasing, then $a_t+b_t \rightarrow -\infty$ and thus $b_t$ must go to negative infinity. We therefore reach a contradiction.

		\item Claim $a_t \rightarrow +\infty$. Proof by contradiction, if $a_t$ bounded by above by $M$, then $a_t+b_t \rightarrow -\infty$, then $a_{t+1} - a_t \geq \eta \big(\frac{1}{2} - \sigma(a_t +b_t) \big) - \eta \sigma'(a_t + b_t) M \geq  \eta \big(\frac{1}{2} - (M+1)\sigma(a_t +b_t) \big) \geq \frac{1}{4}\eta$ for $t$ large enough. We thus reach a contradiction.

		\item Claim $a_t + b_t \rightarrow -\infty$. If $|a_t +b_t| \leq M$ for some absolute constant $M$, then \eqref{eqn:exact-key-2} is a contradiction as we have proved $ra_t - b_t \rightarrow +\infty$.

		\item Claim $\lim\limits_{t \rightarrow \infty} \frac{ra_t - b_t}{t} = \eta r$. We have shown that $\lim_{t\rightarrow \infty} a_{t} \rightarrow \infty$ and $\lim_{t \rightarrow \infty} a_{t} + b_{t} = -\infty$.
	\begin{align*}
		\lim_{t\rightarrow \infty} ~\sigma(a_t) - \sigma(a_t+b_t)  = 1 \;.
	\end{align*} 
	Therefore by \eqref{eqn:exact-key-1} we can show $\lim\limits_{t \rightarrow \infty} \frac{ra_t - b_t}{t} = \eta r$.
		
		\item Claim $\liminf\limits_{t \rightarrow \infty}~ \sigma'(a_t+b_t) (ra_t - b_t) \geq \liminf\limits_{t \rightarrow \infty}~ r\sigma'(a_t+b_t) a_t \geq \frac{r}{1+r}$. Here the last inequality uses the fact $\frac{\sigma'(a_t+b_t) a_t}{\sigma(a_t) - \sigma(a_t+b_t)} \in \left[ \frac{1+\frac{1}{a_t}}{(1+r)+\frac{1}{a_t}}, 1 \right)$, established in Lemma~\ref{lem:key-limit}. 
		
		\item Claim $\limsup\limits_{t\rightarrow \infty} \frac{|a_t+ b_t|}{\log (t)} \leq 1$. This fact is immediate because of $\sigma'(a_t+b_t) (ra_t - b_t) =\Omega(1)$ and $ra_t - b_t = \Theta(t)$.
		
		\item Claim $a_t = \frac{ra_t -b_t }{r+1} + \frac{a_t+b_t}{r+1} = \Theta(t) - O(\log(t)) = \Theta(t)$. Similarly, we can show $-b_t = \Theta(t)$.
	\end{enumerate}
\end{proof}

\subsection{Technical Lemmas for Section~\ref{sec:blessing-curse} }
This section contains the technical building blocks for proofs in Section~\ref{sec:blessing-curse}. Recall $\sigma(z) = 1/(1+e^{-z})$ is the sigmoid function and $\sigma'(z)=\sigma(z)(1-\sigma(z))$.
\begin{lemma}
	\label{lem:helper-function}
	Define $G_{\delta}(x, z) : = -(1+\delta) \sigma'(z) x + \sigma(x) - \sigma(z)$, for $\delta\geq 0$, $z< 0$ and $x \geq 0$. Then,
	\begin{itemize}
		% \item Fix $z < 0$ and let $\delta=0$. If $ x < e^{-z} - e^{z}$, then $G_{0}(x, z) > 0$.
		\item Fix $z<0$ and $\delta\in [0, \infty)$. If $x > \frac{1}{1+\delta}(e^{-z} + 1)$, then $G_{\delta}(x, z) < 0$.
		\item Fix $z<0$ and $\delta \in [0, 1]$, If $x < \frac{1}{1+\delta} (e^{-z} - e^{z})$, then $G_{\delta}(x, z) > 0$.
		% \item Fix $z>0$ and $\delta\in [0, \infty)$. If $x>0$, then $G_{\delta}(x, z) < 0$.
	\end{itemize}
\end{lemma}
\begin{proof}[Proof of Lemma~\ref{lem:helper-function}]
	$G_{\delta}(x, z)$ is a concave function in $x$ for $x \geq 0$ that only crosses the $x$-axis once. Denote $x_0>0$ to be the unique solution to $G_{\delta}(x_0, z) = 0$ for a fixed $z<0$, then
	\begin{align*}
		(1+\delta) \sigma'(z) x_0 = \sigma(x_0) - \sigma(z) \leq 1 - \sigma(z) \;.
	\end{align*}
	We know $x_0 \leq \frac{1}{1+\delta} \frac{1-\sigma(z)}{\sigma'(z)} $ and first claim is thus established.
	
	The second claim follows as
	\begin{align*}
		G_{\delta}(-z, z) \geq 0\;, ~\forall 0\leq \delta \leq 1 \;.
	\end{align*}
	This is because $\inf_{z \leq 0} \frac{\sigma(-z) - \sigma(z)}{-\sigma'(z) z} \geq 2 > 1+\delta$.
	Denote $x_0>0$ to be the unique solution to $G_{\delta}(x_0, z) = 0$, we just proved $x_0 > -z$ and therefore
	\begin{align*}
	 \sigma(-z) - \sigma(z) \leq  \sigma(x_0) - \sigma(z) = (1+\delta) \sigma(z)[1-\sigma(z)] x_0 \;.
	\end{align*}
	Therefore $x_0 > \frac{1}{1+\delta} \frac{\sigma(-z) - \sigma(z)}{\sigma'(z)}$. Rearrange the terms, we have proved $x_0 > \frac{1}{1+\delta} (e^{-z} - e^{z})$. 
\end{proof}

\begin{lemma}
	\label{lem:key-limit}
	Fix $r>0$ and $\eta<1/2(2+r)$. Consider the recursive relationship defined in Lemma~\ref{lem:nonlinear-recur}.
	Assume there exists some $t_0 \in \mathbb{N}$,  such that $(a_{t_0}, b_{t_0})$ satisfy Assumption~\ref{asmp:init-condition} with $r>0$ and $c>1$.
	 Then for $t \geq t_0$, we have
	\begin{itemize}
		\item $a_{t+1}>a_{t}$ is monotonic increasing ;
		\item $a_{t+1} + b_{t+1} < a_t + b_t$ is monotonic decreasing ;
	\end{itemize}
	and 
	\begin{align*}
		 \frac{\sigma'(a_t+b_t) a_t}{\sigma(a_t) - \sigma(a_t+b_t)} \in \left[ \frac{1+\frac{1}{a_t}}{(1+r)+\frac{1}{a_t}}, 1 \right) \;,~ \forall t \geq t_0 \;.
	\end{align*}
\end{lemma}
\begin{proof}[Proof of Lemma~\ref{lem:key-limit}]
	The proof relies on induction. By assumption at $t=t_0$, we know
	\begin{align*}
		\frac{e^{a_t+b_t} a_t}{1 - e^{2(a_t+b_t)}} < 1 \;,~ \frac{e^{a_t+b_t} a_t}{1+ e^{a_t+b_t}} &\geq \frac{1+\frac{1}{a_t}}{ 1+r + \frac{1}{a_t}} \;.
	\end{align*}
	Therefore there exists some fixed small $\epsilon > 0$ that satisfies
	\begin{align}
		\frac{e^{a_t+b_t} a_t}{1 - e^{2(a_t+b_t)}} \leq \frac{1}{1+\epsilon} \;,~ \frac{e^{a_t+b_t} a_t}{1+ e^{a_t+b_t}} &\geq \frac{1+\frac{1}{a_t}}{ 1+r + \frac{1}{a_t}} \;. \label{eqn:eps-room}
	\end{align} 
	Apply Lemma~\ref{lem:helper-function}, we know $G_{0}(a_t, a_t+b_t) > G_{\epsilon}(a_t, a_t+b_t) \geq 0$, and $G_{\frac{r}{1+1/a_t}} (a_t, a_t+b_t) <0$, which implies
	\begin{align*}
		\frac{\sigma'(a_t+b_t) a_t}{\sigma(a_t) - \sigma(a_t+b_t)} \in [ \frac{1+\frac{1}{a_t}}{1+r+\frac{1}{a_t}}, \frac{1}{1+\epsilon} ] \subset [ \frac{1}{1+r}, 1)\;, ~\text{for $t=t_0$} \;.
	\end{align*}
	Furthermore, monotonicity can be established for $t = t_0$,
	\begin{align}
		a_{t+1} &= a_{t} +  \eta G_{0}(a_t, a_t+b_t) > a_{t} >c \;, \label{eqn:an}\\
		a_{t+1} + b_{t+1}& = a_t + b_t + \eta G_{r} (a_t, a_t+b_t) < a_t + b_t + \eta G_{\frac{r}{1+1/a_t}} (a_t, a_t+b_t) < a_{t} + b_{t} \;. \label{eqn:anbn}
	\end{align}
		
	Now for the induction step from $t$ to $t+1$, we apply the Lemma~\ref{lem:recur-estimate}. The assumptions for Lemma~\ref{lem:recur-estimate} are verified, and therefore
	\begin{align*}
		\frac{e^{a_{t+1}+b_{t+1}} a_{t+1}}{1 - e^{2(a_{t+1}+b_{t+1})}} < \frac{e^{a_t+b_t} a_t}{1 - e^{2(a_t+b_t)}} \leq \frac{1}{1+\epsilon} \;, ~\frac{ e^{a_{t+1}+b_{t+1}} a_{t+1}}{1+ e^{a_{t+1}+b_{t+1}}} &> \frac{1+\frac{1}{a_{t+1}}}{ 1+r + \frac{1}{a_{t+1}}} \;. 
	\end{align*}
	Repeat the induction step we can complete the proof. 
\end{proof}

\begin{lemma}
	\label{lem:recur-estimate}
	 Fix $r>0$ and $\eta<1/2(2+r)$. Consider the one-step recursive relationship defined in Lemma~\ref{lem:nonlinear-recur}. Assume $(a_{t}, b_{t})$ satisfy the Assumption~\ref{asmp:init-condition} with $r>0$ and $c>1$.
	 Then we have $a_{t+1}+b_{t+1} < a_{t} + b_{t}$ and  $a_{t+1} > a_{t}$, and furthermore satisfy the following inequality
	\begin{align}
		\frac{e^{a_{t+1}+b_{t+1}} a_{t+1}}{1 - e^{2(a_{t+1}+b_{t+1})}} &< \frac{e^{a_t+b_t} a_t}{1 - e^{2(a_t+b_t)}} \;, \label{eqn:upp} \\
		\frac{ e^{a_{t+1}+b_{t+1}} a_{t+1}}{1+ e^{a_{t+1}+b_{t+1}}} &> \frac{1+\frac{1}{a_{t+1}}}{ 1+r + \frac{1}{a_{t+1}}} \;. \label{eqn:low} 
	\end{align}
	In other words, $(a_{t+1}, b_{t+1})$ satisfy Assumption~\ref{asmp:init-condition} as well.
\end{lemma}
\begin{proof}[Proof of Lemma~\ref{lem:recur-estimate}]
	The proof consists of two parts: a recursive estimate for the upper bound to prove \eqref{eqn:upp} and then a recursive estimate for the low bound in \eqref{eqn:low}. First, due to \eqref{eqn:an} and \eqref{eqn:anbn}, we know $a_{t+1}+b_{t+1} < a_{t} + b_{t}$ and  $a_{t+1} > a_{t}$.
	
	\paragraph{Upper bound recursion}
	Now let's establish the recursion for the upper bound. Note
	\begin{align*}
		&\frac{e^{a_{t+1}+b_{t+1}} a_{t+1}}{1 - e^{2(a_{t+1}+b_{t+1})}} < \frac{e^{a_{t+1}+b_{t+1}} a_{t+1}}{1 - e^{2(a_{t}+b_{t})}}   & \because a_{t+1}+b_{t+1} < a_{t}+ b_{t} \\
		& = \frac{e^{a_t+b_t} a_t}{1 - e^{2(a_t+b_t)}} e^{ \eta G_{r} (a_t, a_t+b_t)  } (1+ \eta \frac{G_0(a_t, a_t+b_t)}{a_t}) \\
		& \leq   \frac{e^{a_t+b_t} a_t}{1 - e^{2(a_t+b_t)}} e^{ \eta G_{r} (a_t, a_t+b_t)  + \eta \frac{G_0(a_t, a_t+b_t)}{a_t} }  & \because 1+z \leq e^{z}  \\
		& =   \frac{e^{a_t+b_t} a_t}{1 - e^{2(a_t+b_t)}} e^{\eta (1+\frac{1}{a_t}) G_{\frac{r}{1+1/a_t}} (a_t, a_t+b_t) } & \text{by definition of $G_{\theta}(x, z)$} \\
		& <  \frac{e^{a_t+b_t} a_t}{1 - e^{2(a_t+b_t)}} < 1 & \because G_{\frac{r}{1+1/a_t}} (a_t, a_t+b_t) < 0\;.
	\end{align*}
	In addition, we can prove that there exists a constant $\epsilon >0$ as in \eqref{eqn:eps-room} such that $G_{\epsilon}(a_t, a_t+b_t) > 0$ and
	\begin{align*}
		a_{t+1} - a_{t} = \eta (\sigma(a_t) - \sigma(a_t+b_t)) [ 1 - \frac{\sigma'(a_t+b_t) a_t}{\sigma(a_t) - \sigma(a_t+b_t)} ] > \eta (\sigma(c) - \sigma(0)) \frac{\epsilon}{1+\epsilon} = \Omega(1).
	\end{align*}
	
	\paragraph{Lower bound recursion} Define 
	\begin{align*}
		\underline{r_{t}}:= \frac{1+\frac{1}{a_t}}{ 1+r + \frac{1}{a_t}}, ~\forall t \;.
	\end{align*}
	Define $\tilde{r}_{t} : = \sigma(a_t+b_t) a_t$, we are going to establish $\tilde{r}_{t+1} > \underline{r_{t+1}}$.
	
	Now let's establish the recursion for the lower bound. 
	\begin{align*}
		&\frac{ e^{a_{t+1}+b_{t+1}} a_{t+1}}{1+ e^{a_{t+1}+b_{t+1}}} > \frac{ e^{a_{t+1}+b_{t+1}} a_{t+1}}{1+ e^{a_{t}+b_{t}}} & \because a_{t+1}+b_{t+1} < a_{t}+ b_{t}\\
		&=\frac{e^{a_t+b_t} a_t}{1+ e^{a_{t}+b_{t}}}   e^{ \eta G_{r} (a_t, a_t+b_t)  } e^{\log(1+ \frac{a_{t+1} - a_t}{a_t})} \\
		& > \frac{e^{a_t+b_t} a_t}{1+ e^{a_{t}+b_{t}}}  e^{ \eta G_{r} (a_t, a_t+b_t)  } e^{ \eta \frac{G_0(a_t, a_t+b_t)}{a_{t+1}} }  &  \because \log(1+z) > \frac{z}{1+z} \\
		& =  \sigma(a_t+b_t) a_t \cdot e^{\eta \big( - (1+r+\frac{1}{a_{t+1}}) \sigma'(a_t+b_t) a_t + (1+\frac{1}{a_{t+1}}) [\sigma(a_t) - \sigma(a_t+b_t) ] ) \big)} \;.
	\end{align*}
	Recall the definition of $\tilde{r}_{t} : = \sigma(a_t+b_t) a_t$, we continue
	\begin{align*}
		& =  \sigma(a_t+b_t) a_t \cdot e^{\eta \big( - (1+r+\frac{1}{a_{t+1}}) \sigma'(a_t+b_t) a_t + (1+\frac{1}{a_{t+1}}) [\sigma(a_t) - \sigma(a_t+b_t) ] ) \big)} \\
		& =  \tilde{r}_{t} \cdot e^{\eta \big( - (1+r+\frac{1}{a_{t+1}}) \sigma'(a_t+b_t) a_t + (1+\frac{1}{a_{t+1}}) [1 - \sigma(a_t+b_t) ] ) \big)} e^{-\eta (1+ \frac{1}{a_{t+1}}) (1-\sigma(a_t))}  \\
		& = \tilde{r}_{t} \cdot e^{\eta [1 - \sigma(a_t+b_t) ] \big( - (1+r+\frac{1}{a_{t+1}}) \sigma(a_t+b_t) a_t + (1+\frac{1}{a_{t+1}})  ) \big)} e^{-\eta (1+ \frac{1}{a_{t+1}}) (1-\sigma(a_t))}  \\
		& > \tilde{r}_{t} \cdot \left\{ 1 -  \eta[1-\sigma(a_t+b_t)] (1+r+\frac{1}{a_{t+1}})  \big(  \tilde{r}_{t}  - \frac{1+\frac{1}{a_{t+1}}}{1+r+\frac{1}{a_{t+1}}}\big)  \right\} e^{-\eta (1+ \frac{1}{a_{t+1}}) (1-\sigma(a_t))}  \;.
	\end{align*}
	Here the last two lines follow from the facts $\sigma'(z) = \sigma(z) (1-\sigma(z))$ and $e^z \geq 1+z$.
	We first control the last term on the RHS of the above equation,
	\begin{align*}
		e^{-\eta (1+ \frac{1}{a_{t+1}}) (1-\sigma(a_t))} \geq 1 - \eta (1+ \frac{1}{a_{t+1}}) (1-\sigma(a_t)) >1 - \eta \frac{1+ c^{-1}}{1+e^{a_t}} > 1 - \eta (1+ c^{-1}) e^{-a_t} \;.
	\end{align*}
	Define $\tilde{\eta}  := \eta(1+r+c^{-1}) > \eta[1-\sigma(a_t+b_t)] (1+r+\frac{1}{a_{t+1}}) $, we continue with the bound
	\begin{align}
		\tilde{r}_{t+1} > \tilde{r}_{t} \left[ 1 - \tilde{\eta}  \big(  \tilde{r}_{t}  - \underline{r_{t+1}}\big)  \right] \left[ 1 - \eta (1+ c^{-1}) e^{-a_t} \right] \label{eqn:r-lower-bound} \;.
	\end{align}
	For $a_t > c$ some absolute constant, we have $e^{-a_t}<\underline{r_t} - \underline{r_{t+1}}$, as the RHS $ \underline{r_t} - \underline{r_{t+1}} $ is of order $(a_{t+1} - a_{t})a_t^{-2} = \Theta(a_t^{-2}) $ yet the LHS is exponential in $a_t$. Therefore $\eta(1+ c^{-1}) e^{-a_t} < \tilde{\eta}e^{-a_t} < \tilde{\eta}\big( \underline{r_t} - \underline{r_{t+1}} \big) <  \tilde{\eta}  \big(  \tilde{r}_{t}  - \underline{r_{t+1}}\big) $, and thus \eqref{eqn:r-lower-bound} reads $$\tilde{r}_{t+1} > \tilde{r}_{t} \left[ 1 - \tilde{\eta}  \big(  \tilde{r}_{t}  - \underline{r_{t+1}}\big)  \right] ^2 \;.$$
	One can subtract $\underline{r_{t+1}}$ on both sides, and see
	\begin{align*}
		\tilde{r}_{t+1} - \underline{r_{t+1}} >  \tilde{r}_{t} \left[ 1 - \tilde{\eta}  \big(  \tilde{r}_{t}  - \underline{r_{t+1}}\big)  \right]^2 -  \underline{r_{t+1}} \geq (\tilde{r}_{t}  -  \underline{r_{t+1}})[ 1- 2\tilde{\eta} \tilde{r}_{t} +\tilde{\eta}^2 \tilde{r}_{t} (\tilde{r}_{t}  -  \underline{r_{t+1}}) ]  \;.
	\end{align*}
	The lower bound on Eqn.~\eqref{eqn:low} is true as long as the RHS of the above is positive
	\begin{align*}
		(\tilde{r}_{t}  -  \underline{r_{t+1}})[ 1- 2\tilde{\eta} \tilde{r}_{t} +\tilde{\eta}^2 \tilde{r}_{t} (\tilde{r}_{t}  -  \underline{r_{t+1}})] > 0 \;.
	\end{align*}
	Note $\tilde{\eta} = \eta(1+r+c^{-1}) < \eta (2+r)<1/2$, $\tilde{r}_t < 1$ and $\tilde{r}_t > \underline{r_{t}} > \underline{r_{t+1}}$, we have concluded
	\begin{align*}
		\frac{ e^{a_{t+1}+b_{t+1}} a_{t+1}}{1+ e^{a_{t+1}+b_{t+1}}}=\tilde{r}_{t+1} >\underline{r_{t+1}} = \frac{1+\frac{1}{a_{t+1}}}{ 1+r + \frac{1}{a_{t+1}}}  \;.
	\end{align*}
\end{proof}

\subsection{Proofs in Section~\ref{sec:impact-on-learner}}
\begin{proof}[Proof of Theorem~\ref{thm:learner-reaction-regression}]
	Let $y^i = \langle x^i, \theta^\star \rangle + \epsilon^i$ with $\epsilon^i$ i.i.d. Gaussian, then
	 \begin{align*}
	 	\theta^{(1)} &= \theta^{(0)} - \eta \cdot \frac{1}{n} \sum_{i=1}^n \pdv{\theta} \ell(\langle x^i, \theta \rangle, y^i)|_{\theta = \theta^{(0)}} \;,\\
		\theta^{(1)} - \theta^\star &= \left(I - 2\eta \frac{1}{n} \sum_{i=1}^n x^i \otimes x^i \right) (\theta^{(0)} - \theta^\star) + 2\eta  \frac{1}{n}  \sum_{i=1}^n \epsilon^i x^i \;.
	 \end{align*}
	 Fix $n$, first let $T \rightarrow \infty$. Theorem~\ref{thm:regression-main} shows that $x^i_{T} \rightarrow \Delta_{\mathrm{b}}$ (we denote the dependence on $T$ explicitly) in the following sense
	 \begin{align*}
	 	\langle x^i_{T}, \Delta_{\mathrm{b}} \rangle^2 \geq 1 - O(e^{-cT}) \;.
	 \end{align*}
	 Choose $\eta = 1/2$, we have
	 \begin{align*}
	 	\| \theta^{(1)} - \theta^\star \| \leq \| \underbrace{(I -  \frac{1}{n} \sum_{i=1}^n x^i \otimes x^i ) (\theta^{(0)} - \theta^\star)}_{:= \Gamma} \| + \| \frac{1}{n}  \sum_{i=1}^n \epsilon^i x^i \| \;.
	 \end{align*}
	 
	 For the first term $\Gamma$, note $\Gamma = \langle \Delta_{\mathrm{b}}, \Gamma \rangle \Delta_{\mathrm{b}} +  \Pi^{\perp}_{\Delta_{\mathrm{b}}} \Gamma $. Along the direction $\Delta_{\mathrm{b}}$,
	 \begin{align*}
	 	\langle \Delta_{\mathrm{b}}, \Gamma \rangle &= \| \theta^{(0)} - \theta^\star \| \cdot (1 - \frac{1}{n} \sum_{i=1}^n \langle x^i, \Delta_{\mathrm{b}} \rangle^2 ) \;.
	 \end{align*}
	 Projecting to the orthogonal complement to $\Delta_{\mathrm{b}}$, we know
	 \begin{align*}
	 	\| \Pi^{\perp}_{\Delta_{\mathrm{b}}} \Gamma \| \leq   \frac{1}{n} \sum_{i=1}^n  \langle x^i, \Delta_{\mathrm{b}} \rangle \| \Pi^{\perp}_{\Delta_{\mathrm{b}}} x_i  \| \leq  \frac{1}{n} \sum_{i=1}^n  \langle x^i, \Delta_{\mathrm{b}} \rangle \sqrt{1 - \langle x^i, \Delta_{\mathrm{b}} \rangle^2} \;.
	 \end{align*}
	 Therefore
	 \begin{align*}
	 	\lim_{T\rightarrow \infty}~ \| \Gamma \| \leq \lim_{T\rightarrow \infty}~ \left\{\| \theta^{(0)} - \theta^\star \| \cdot (1 - \frac{1}{n} \sum_{i=1}^n \langle x^i_T, \Delta_{\mathrm{b}} \rangle^2 )  + \frac{1}{n} \sum_{i=1}^n  \langle x^i_T, \Delta_{\mathrm{b}} \rangle \sqrt{1 - \langle x^i_T, \Delta_{\mathrm{b}} \rangle^2} \right\}  = 0 \;.
	 \end{align*}
	 For the second term, we apply the Hanson--Wright inequality and recall $\|x^i\|=1$ to reach
	 {\small
	 \begin{align*}
	 	\left\| \frac{1}{n}  \sum_{i=1}^n \epsilon^i x^i \right\| \leq \sqrt{C_1 \cdot \frac{ \mathrm{Tr}( \tfrac{\sum_{i=1}^n x^i \otimes x^i}{n}) (1+\log^{0.5}(1/\delta))  +  C_2\| \tfrac{\sum_{i=1}^n x^i \otimes x^i}{n} \|_{\rm op}  \log (1/\delta)}{n}} = O(\sqrt{\frac{\log(1/\delta)}{n}}) 
	 \end{align*}
	 }
	 with probability at least $1-\delta$. 
\end{proof}

\begin{proof}[Proof of Theorem~\ref{thm:learner-reaction-classification}]
	In what follows, we consider fixed $n$, and let $T \rightarrow \infty$.
 \begin{align*}
 	\theta^{(1)} &= \theta^{(0)} - \eta \cdot \frac{1}{n} \sum_{i=1}^n \pdv{\theta} \ell(\langle x^i, \theta \rangle, y^i)|_{\theta = \theta^{(0)}} \;,\\
	\theta^{(1)} - \theta^\star &=  (\theta^{(0)} - \theta^\star) - \eta  \frac{1}{n}  \sum_{i=1}^n \big(\sigma(\langle x^i,  \theta^{(0)} \rangle)-y^i \big) x^i \;.
 \end{align*}
 Theorem~\ref{thm:regression-main} shows that $x^i_{T} \rightarrow \Delta_{\mathrm{c}}$, where the convergence means directional convergence, and we denote the dependence on $T$ explicitly. Then as $\lim_{T\rightarrow \infty} ~\langle \Delta_{\mathrm{c}}, x^{i}_{T} \rangle = 1$ and $\| x^{i}_{T} \| = 1$, we know, 
\begin{align*}
 	\lim_{T\rightarrow \infty} ~\langle \tfrac{\theta^\star}{\| \theta^\star\|}, x^i_{T} \rangle^2 = 1 - \lim_{T\rightarrow \infty} ~\langle \Delta_{\mathrm{c}}, x^i_{T} \rangle^2 = 0 \;,
\end{align*}
 and therefore
 \begin{align*}
 	\lim_{T\rightarrow \infty}~ \left| \langle \theta^\star - \theta^{(1)}_{n, T, \eta}, \tfrac{\theta^\star}{\| \theta^\star\|} \rangle - \langle \theta^\star - \theta^{(0)}, \tfrac{\theta^\star}{\| \theta^\star\|} \rangle \right| & \leq  \eta  \frac{1}{n}  \sum_{i=1}^n \lim_{T\rightarrow \infty} | \sigma(\langle x^i_{T},  \theta^{(0)} \rangle)-y^i_{T} \big|  \cdot \big| \langle \tfrac{\theta^\star}{\| \theta^\star\|} , x^i_{T} \rangle \big| \rightarrow 0\;.
 \end{align*}
 The final claim is a direct fact from
 \begin{align*}
 	\lim_{T\rightarrow \infty}~  \| \theta^\star - \theta^{(1)}_{n, T, \eta} \| \geq \left| \langle \theta^\star - \theta^{(1)}_{n, T, \eta}, \tfrac{\theta^\star}{\| \theta^\star\|} \rangle \right| = \left| \langle \theta^\star - \theta^{(0)}, \tfrac{\theta^\star}{\| \theta^\star \|} \rangle \right| \;.
 \end{align*}
 We conclude the proof by taking $\liminf$ w.r.t. to $n$.
\end{proof}

\end{document}